\documentclass{article}

\usepackage{microtype}
\usepackage{graphicx}
\usepackage{subcaption}
\usepackage{booktabs} 
\usepackage[preprint]{icml2026}
\usepackage{amsmath}
\usepackage{tabularx}

\usepackage{multirow}
\usepackage{hyperref}
\usepackage{url}
\usepackage[utf8]{inputenc} 
\usepackage[T1]{fontenc}    
\usepackage{amsfonts}       
\usepackage{nicefrac}       
\usepackage{xcolor}         
\usepackage{amsthm}
\usepackage{mathtools,mathrsfs,bm}
\usepackage{amssymb}
\usepackage{accents}
\usepackage{tikz}
\usepackage{enumitem,comment}
\usepackage{wrapfig}
\usepackage{CJKutf8}
\usepackage{hyperref}
\allowdisplaybreaks

\usepackage[capitalize,noabbrev]{cleveref}

\newtheorem{thm}{Theorem}
\newtheorem{lemma}[thm]{Lemma} 

\theoremstyle{definition} 
\newtheorem{defn}{Definition}

\setenumerate[1]{leftmargin=*, topsep=0pt, itemsep=0pt, label={\upshape(\arabic*)}}

\newcommand{\norm}[1]{\lVert#1\rVert}

\newcommand{\abs}[1]{\left\lvert#1\right\rvert}

\newcommand{\tc}[1]{\left\langle#1\right\rangle} 

\newcommand\bx{\bm{x}}
\newcommand\bc{\bm{c}}

\newcommand\by{\bm{y}}
\newcommand\bp{\bm{p}}
\newcommand\be{\bm{e}}
\newcommand\bz{\bm{z}}
\newcommand\bv{\bm{v}}
\newcommand\bh{\bm{h}}

\newcommand\bw{\bm{w}}
\newcommand\bu{\bm{u}}

\newcommand\bW{\mathbf{W}}
\newcommand\bK{\mathbf{K}}
\newcommand\bQ{\mathbf{Q}}
\newcommand\bV{\mathbf{V}}

\DeclareMathOperator{\TF}{TF}

\DeclareMathOperator{\Unif}{Unif}

\DeclareMathOperator{\RR}{\mathbb{R}}

\DeclareMathOperator{\NN}{\mathbb{N}}

\newcommand\sm{\mathsf{softmax}}
\newcommand{\B}[1]{\boldsymbol{#1}_n}

\theoremstyle{plain}

\theoremstyle{definition}

\theoremstyle{remark}

\usepackage[utf8]{inputenc}
\usepackage{tcolorbox}
\usepackage{xcolor}
\usepackage{array}
\usepackage[textsize=tiny]{todonotes}

\newtcolorbox{databox}[1][]{
  colback=white,
  colframe=black,
  coltitle=white,
  fonttitle=\bfseries,
  title=#1,
  sharp corners,
  boxrule=0.8pt,
  left=5pt, right=5pt, top=5pt, bottom=5pt
}

\begin{document}

\twocolumn[
   \icmltitle{Chain Of Thought Compression: A Theoretical Analysis}




  \icmlsetsymbol{equal}{*}

  \begin{icmlauthorlist}
    \icmlauthor{Juncai Li}{sxu}
    \icmlauthor{Ru Li}{sxu}
    \icmlauthor{Yuxiang Zhou}{queen}
    \icmlauthor{Boxiang Ma}{sxu}
    \icmlauthor{Jeff Z. Pan}{edb}
  \end{icmlauthorlist}

  \icmlaffiliation{sxu}{School of Computer and Information Technology, Shanxi University, Taiyuan, Shanxi, China}
  \icmlaffiliation{queen}{Queen Mary, University of London, UK}
  \icmlaffiliation{edb}{School of Informatics, University of Edinburgh, UK}

  \icmlcorrespondingauthor{Ru Li}{liru@sxu.edu.cn}

  \icmlkeywords{Machine Learning, ICML}

  \vskip 0.3in
]



\printAffiliationsAndNotice{}  

\begin{abstract}

  Chain-of-Thought (CoT) has unlocked advanced reasoning abilities of Large Language Models (LLMs) with intermediate steps, yet incurs prohibitive computational costs due to generation of extra tokens.
  Recent studies empirically show that compressing reasoning steps into latent states, or implicit CoT compression, offers a token-efficient alternative.
  However, the mechanism behind CoT compression remains unclear.
  In this paper, we provide the first theoretical analysis of the difficulty of learning to internalize intermediate reasoning steps. 
  By introducing \emph{Order-$r$ Interaction}, we prove that the learning signal for high-order logical dependencies exponentially decays to solve \emph{irreducible problem}, where skipping intermediate steps inevitably leads to high-order interaction barriers.
  To empirically validate this, we introduce NatBool-DAG, a challenging benchmark designed to enforce irreducible logical reasoning and eliminate semantic shortcuts. 
  Guided by our theoretical findings, we propose ALiCoT (\textbf{Al}igned \textbf{I}mplicit \textbf{CoT}), a novel framework that overcomes the signal decay by aligning latent token distributions with intermediate reasoning states. 
  Experimental results demonstrate that ALiCoT successfully unlocks efficient reasoning: it achieves a \textbf{54.4$\times$ speedup} while maintaining performance comparable to explicit CoT.
\end{abstract}

\section{Introduction}

In recent years, Large Language Models (LLMs) have achieved a paradigm shift in reasoning capabilities, evolving from simple pattern matching to addressing complex tasks through Chain-of-Thought (CoT) reasoning~\cite{wei2022chain, zhao2023large,ma-etal-2025-memorization-understanding}, generating sequential intermediate steps leading to the final answer.
While thinking models such as GPT-4-o1 \cite{gpto1} and DeepSeek-R1 \cite{guo2025deepseek} have achieved remarkable sucess in mathematical and commonsense reasoning,
it comes at a steep price: the heavily expanded reasoning chains incurs massive token consumption and high latency, creating a significant bottleneck for real-world deployment and efficiency \cite{zhu2025surveylatentreasoning, warner-etal-2025-smarter}.
To bridge this gap, recent works have pivoted towards CoT Compression. One direction that has been empirically successful is explicitly reducing the length of intermediate steps by training model to learn shorter CoT~\cite{kang2025c3ot}. 
Another distinct line of research has aimed to internalize intermediate reasoning steps to latent states to perform implicit CoT~\cite{zhang2025lightthinkerthinkingstepbystepcompression,hao2024traininglargelanguagemodels,shen2025codicompressingchainofthoughtcontinuous}. 
Meanwhile, several research studies posited that implicit CoT maintain competitive performance on tasks where dependencies are often shallow or linguistically redundant (e.g., semantic shortcuts), yet suffer a precipitous performance collapse on solving tasks requiring rigorous logic such as multi-step mathematical reasoning or symbolic derivation \cite{zhu2025surveylatentreasoning, chen2025reasoning}.
\begin{figure}[t]
  \centering
  \includegraphics[width=0.5\textwidth]{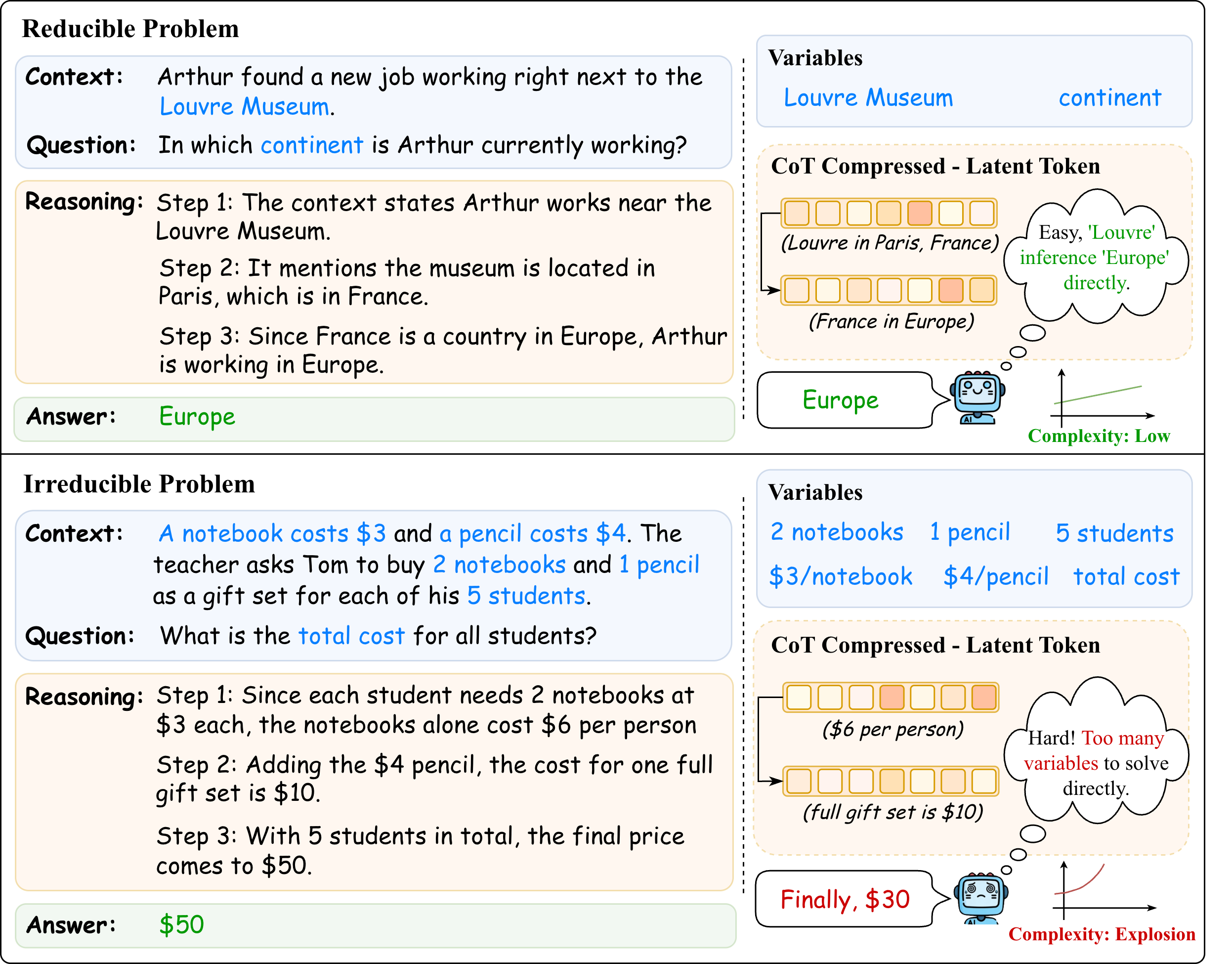}%
  \caption{CoT compression impacts reasoning difficulty based on problem reducibility. For Reducible Problems (top), compression is effective as the model faces low difficulty in reconstructing reasoning relationships. Conversely, for Irreducible Problems (bottom), omitting intermediate steps removes essential computational states, triggering a complexity explosion that leads to failure.}
  \vspace{-2mm}
  \label{fig:cot_comparison}
\end{figure}

Despite these empirical successes, however, the theoretical understanding of the implicit CoT compression mechanism remains limited. Prior investigations into CoT mechanisms~\cite{kimtransformers,abbe2024far} often overlook the critical challenges arising from the omission of local reasoning steps. Consequently, it is unclear how the difficulty of internalization scales with the number of steps being internalized across different tasks.
Motivated by this question, we provide the first theoretical analysis, by introducing \emph{Order-$r$ Interaction}, a formal measure quantifying the complexity of logical dependencies required to bypass intermediate steps, we prove that the learning signal for high-order logical dependencies exponentially decays to solve \emph{irreducible problem}, where skipping intermediate steps inevitably leads to high-order interaction barriers.
For example, as illustrated in Figure \ref{fig:cot_comparison}, the model can successfully shortcut from ``Louvre'' to ``Europe'' via simple associations in latent tokens but fails to skip intermediate steps in arithmetic tasks, where the complex dependencies among variables lead to high-order interaction barriers.

By analyzing the learning dynamics of the Parity Problem, we reveal that \textbf{the difficulty of learning to compress implicit CoT poses a fundamental optimization challenge: the gradient signal required to learn high-order logical dependencies decays exponentially with the number of compressed steps, effectively burying the reasoning path in noise.}

Motivated by this theoretical insight, we present a comprehensive study spanning theory, evaluation, and methodology to overcome these barriers. Our contributions are as follows:

$\bullet$ \textbf{The Intractability of CoT Compression:} We provide a theoretical perspective on why Implicit CoT fails in complex logic. We prove that as reasoning steps are compressed into latent states without supervision, the learning efficiency for high-order interactions degrades exponentially. This theoretical result explains the "reasoning collapse" observed in existing methods when scaling to deep logical tasks.

$\bullet$ \textbf{NatBool-DAG Benchmark:} To rigorously evaluate CoT compression beyond the noise of natural language redundancies, we introduce NatBool-DAG, a scalable synthetic benchmark based on Directed Acyclic Graphs of Boolean logic embedded in natural language. Unlike standard datasets that may allow semantic shortcuts, NatBool-DAG enforces strict, irreducible logical dependencies, serving as a litmus test for genuine implicit reasoning capabilities.

$\bullet$ \textbf{ALiCoT: Aligned Implicit Chain-of-Thought:} Guided by our theory, we propose ALiCoT, a framework imposing distribution alignment on latent tokens. By aligning latent tokens with explicit reasoning semantics, ALiCoT prevents signal decay. Results show significantly improved learning efficiency and stability, matching Explicit CoT on complex tasks while retaining compression benefits.

\section{Related Work}

Reasoning-efficient LLMs generally target explicit CoT compression or implicit latent reasoning. Explicit methods shorten natural language rationales via distillation \cite{kang2025c3ot, luo2025deconstructing} or token-level pruning \cite{cheng2024compressed, li2025compressing, xia-etal-2025-tokenskip}, sometimes emphasizing semantic keypoints to mitigate gradient issues \cite{feng2024keypoint}. However, these discrete approaches remain fundamentally bound by text length constraints. Conversely, Implicit CoT internalizes reasoning into continuous states \cite{deng2023implicit, deng2024explicit} or autoregressive latent scratchpads \cite{hao2024traininglargelanguagemodels, xu2025softcot}. Yet, end-to-end training often suffers from latent collapse, favoring shallow shortcuts over deep reasoning \cite{wei2025sim, zhang2025latent}. While recent frameworks like CODI \cite{shen2025codicompressingchainofthoughtcontinuous} and SIM-CoT \cite{wei2025sim} introduce auxiliary supervision to align latent and explicit representations, they lack rigorous principled theoretical guidance. Consequently, their heuristic constraints often fail to fully capture the intricate logical dependencies required for truly complex tasks.

\begin{figure*}[htbp]
    \centering
    \includegraphics[width=0.9\textwidth]{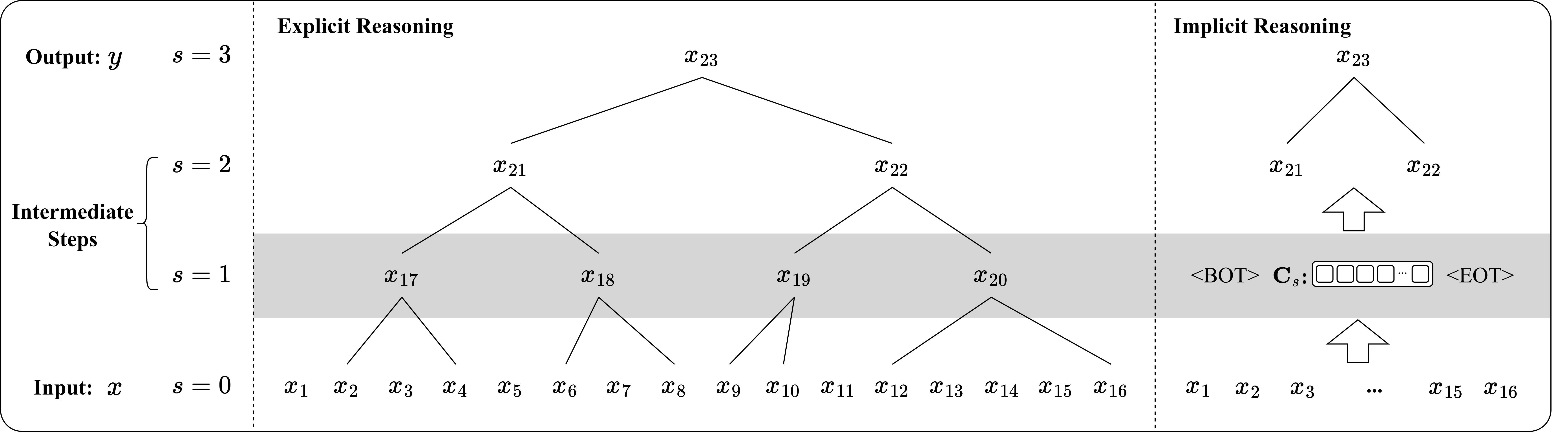}
    \caption{Schematic of the k-Parity Problem. $\bx_{17}, \cdots, \bx_{22}$ represent intermediate reasoning steps, which are replaced by latent tokens in the case of implicit Chain-of-Thought.}
    \label{fig:parity_problem}
\end{figure*}

\section{Preliminaries and Setup}

This section first defines the mathematical notations used in the paper, then systematically introduces the foundational problem (parity problem), core object (Implicit CoT), and model architecture (Transformer) required for subsequent analysis, laying the groundwork for theoretical and experimental discussions.

\textit{Notation}. We denote the set $\{1,2,\ldots,n\}$ as $[n]$. Vectors and matrices are denoted in bold text (e.g., $\boldsymbol{x}, \boldsymbol{A}$), whereas scalars appear in plain text (e.g., $y$). 
For $\bz\in\RR^n$ we write $\phi(\bz) = (\phi(z_1),\cdots, \phi(z_n))^\top$, $\bz^2 = \bz\odot\bz = (z_1^2,\cdots,z_n^2)$ and $\abs{\bz} = (|z_1|,\cdots,|z_n|)^\top$. The 2-norm is always denoted by $\norm{\cdot}$. The multi-linear inner product or contraction of $\bz_1,\cdots,\bz_r\in\RR^n$ for any $r\in\NN$ is denoted as $\tc{\bz_1,\cdots,\bz_r} := \sum_{i=1}^n z_{1,i}\cdots z_{r,i}$. In particular, $\tc{\bz_1} = \bz_1^\top\B1$ and $\tc{\bz_1,\bz_2} = \bz_1^\top\bz_2$.

\subsection{The Parity Problem}
\label{sec:parity_problem}

Given an index set $\{1,\cdots,d\}$, for $d$-bit inputs $\bx = (x_j)_{j=1}^d\sim\Unif(\{\pm 1\}^d)$, the $k$-parity problem involves predicting $y=\prod_{j\in p} x_j$, where $p\subseteq [d]$ and $|p|=k$. In this problem, the support subset $p$ is unknown, 
and we denote the set of all possible $p$ as $P_k$, so $|P_k|=\binom{d}{k}$. Our problem setup follows the definition in \citet{ICLR2025_6e2986de} for theoretical analysis. We abuse notation and identify the set of indices $p$ with the corresponding parity mapping $\bx\mapsto \prod_{j\in p} x_j$. Given $n$ samples $(\bx^i,y^i)_{i\in[n]}$, our goal is to accurately predict the parity of any unseen test input.

\subsection{Transformer Model and Optimization Objective}

We investigate the performance of Transformer models \cite{vaswani2017attention} in solving the parity problem. While standard Transformers consist of multiple stacked self-attention and feed-forward network (FFN) layers, we analyze a simplified architecture comprising a single attention layer and a specialized FFN tailored for parity tasks (see Section \ref{sec:transformer_model} for details). Given an input $\bx \in \mathbb{R}^d$, the model projects it into an embedding space, processes it through the self-attention and FFN modules, and generates a prediction via a Softmax layer.
For optimization, we employ the standard Cross-Entropy (CE) loss. Let $\hat{\by}^i \in [0,1]^2$ denote the predicted probability vector for the $i$-th sample, corresponding to the class labels $\{1, -1\}$. The empirical risk $\mathcal{L}$ over $n$ samples is formulated as:
\begin{equation}
   \label{eqn:dir_ce_loss}
   \mathcal{L} = - \frac{1}{n} \sum_{i=1}^n \sum_{c \in \{1, -1\}} \mathbb{I}(y^i = c) \log (\hat{\by}^i_c),
\end{equation}

where $\mathbb{I}(\cdot)$ is the indicator function and $\hat{\by}^i_c$ represents the probability assigned to label $c$. Minimizing $\mathcal{L}$ is equivalent to maximizing the likelihood of the correct parity label.

\subsection{Chain-of-thought Learning Objectives}

For the parity problem, a Chain-of-Thought reasoning approach decomposes the $k$-parity task into a hierarchical composition of $2$-parity operations, as shown in Figure \ref{fig:parity_problem}. Specifically, given a $d$-bit input $\bx=(x_1, \cdots, x_d) \in \{1, -1\}^d$, the calculation of the parity for a size-$k$ subset can be modeled as a binary tree of height $H=\lceil \log_2 k \rceil$. 

The reasoning process introduces a sequence of intermediate tokens. For a new node indexed by $t > d$, its value is computed by multiplying its two preceding child nodes, denoted as $x_t = x_{u} \cdot x_{v}$, where $u, v < t$ are the indices of the operands from the previous level. For instance, in Figure \ref{fig:parity_problem}, the intermediate step at $s=1$ generates tokens $x_{17}, \dots, x_{20}$ by aggregating pairs from the input. Finally, after $H$ reasoning steps, the root node of the tree represents the final scalar parity output $y$.

Thus, by supervising the intermediate steps, the model can progressively learn to solve each $2$-parity subproblem, ultimately deriving the solution for the entire $k$-parity problem. The training objective of the model then becomes minimizing the cross-entropy loss at each step:
\begin{equation}
  \label{eq:cot_ce_loss}
  \begin{split}
   \mathcal{L}^{\text{CoT}}(\theta) 
    &= -\frac{1}{n} \sum_{j=d+1}^{d+k-1} \sum_{c \in \{1, -1\}} \tc{\mathbb{I}(\bx_j=c), \log(\hat \bx_{j,c})},
  \end{split}
\end{equation}

\subsection{Implicit Chain-of-Thought}

A common approach to Implicit Chain-of-Thought (ImpCoT) is to gradually replace explicit reasoning steps with "thinking tokens" \citep{zhang2025lightthinkerthinkingstepbystepcompression, hao2024traininglargelanguagemodels}. Rather than abandoning the chain structure entirely, this method aims to compress reasoning steps into compact latent representations.

In the context of the parity problem, we implement this by concealing an entire hierarchical layer of the solution tree at a time. As shown in Figure \ref{fig:parity_problem}, when Step 1 (layer 1) is converted into an implicit chain, the entire sequence of explicit intermediate tokens corresponding to that layer (e.g., $[\bx_{17},\cdots,\bx_{20}]$) is replaced by a single latent token $\bc_s$. Consequently, the model must learn to encode the computational logic of that specific layer into this dense vector representation.

When step $s$ is implicit, the training objective transforms into predicting the subsequent explicit steps (starting from step $s+1$) given the input $\bx$ and the latent token $\bc_s$. The loss function skips the supervision for the concealed layer and is formulated as:

\begin{equation}
  \label{eq:imp_cot_loss}
  \begin{split}
   \mathcal{L}^{\text{iCoT}}(\theta) &= -\frac{1}{n} \sum_{j=d+\tau_s+1}^{d+k-1} \sum_{c \in \{1, -1\}} \tc{\mathbb{I}(\bx_j=c), \log(\hat \bx_{j,c})}, \\
   \text{with } \hat{\bx}_j &= f_\theta(\bx_{1:d}, \bc_s, \bx_{d+\tau_s+1:j-1}),
  \end{split}
\end{equation}
where $\tau_s$ denotes the total number of explicit tokens originally required up to step $s$, and $f_\theta$ predicts the target token $\bx_j$ based on the problem input $\bx_{1:d}$, the latent state $\bc_s$, and the history of subsequent explicit steps.

\textbf{Embedding Mechanisms for Latent Tokens.} 
Existing approaches for implementing implicit reasoning typically employ special control tokens—such as \texttt{<bot>} (Begin of Thought) and \texttt{<eot>} (End of Thought)—to delimit the latent space. Under this common framework, the reasoning process is structured as a sequence $[\bx_{\text{context}}, \texttt{<bot>}, \bc_s, \texttt{<eot>}]$. However, the derivation of the latent vector $\bc_s$ varies significantly across different studies, generally falling into two paradigms:

\textbf{$\bullet$Imp.Base-1: Independent Learnable Parameters.} 
One category of methods treats the latent token $\bc_s$ as a standalone, global parameter \cite{zhang-etal-2025-lightthinker, goyal2024think}. In this setting, $\bc_s$ is decoupled from the model's immediate runtime activations and functions similarly to a "soft prompt" or a continuous instruction embedding. The vector is typically initialized by sampling from a high-dimensional uniform distribution:
\begin{equation}
   \bc_s \sim \mathcal{U}(-1, 1)^{d_{\text{model}}}.
\end{equation}
During training, this vector is optimized via backpropagation to encode a static, generalized reasoning pattern that bridges the gap between the context and the answer, remaining constant across different inference instances.

\textbf{$\bullet$Imp.Base-2: Dynamic Generation via Internal States.} 
The other category generates $\bc_s$ dynamically from the model's internal processing\cite{hao2024traininglargelanguagemodels}, enabling an instance-specific "thinking" flow. This process is autoregressive:
\begin{enumerate}
    \item \textbf{State Extraction:} Given the context ending with the start marker $[\dots, \bx_{t}, \texttt{<bot>}]$, the model performs a forward pass. The output hidden state at the \texttt{<bot>} position is captured to serve as the latent token:
    \begin{equation}
          \bc_s \leftarrow \bh_{\texttt{<bot>}}.
    \end{equation}
    \item \textbf{Recurrent Inference:} This computed state $\bc_s$ is then fed back into the sequence followed by \texttt{<eot>}. The model processes $[\dots, \texttt{<bot>}, \bc_s, \texttt{<eot>}]$ to predict the subsequent explicit token.
\end{enumerate}
Unlike the static parameter approach, this mechanism ensures that the latent representation is dynamically constructed based on the specific input context.

\section{Theoretical results}

Through a theoretical analysis of the parity problem, this section first elucidates the learning dynamics governing how Transformers capture critical contextual information to perform reasoning. Building on this foundation, we contrast the learning efficuiency of implicit versus explicit CoT, quantify the training overhead of implicit approaches, and identify the limitations inherent in two dominant implicit CoT training paradigms. Finally, we unveil why current implicit CoT techniques yield performance gains on specific datasets yet remain ineffective for complex tasks such as mathematical reasoning.

\subsection{Order-$r$ Interaction and Signal Decomposition}
\label{sec:4.1}

To rigorously quantify reasoning mechanics, we model logical dependencies as high-order variable interactions within the context of the parity problem. Here, complexity corresponds to the interaction order, the size of the irreducible input subset determining the output. This distinguishes reasoning from simple pattern matching, as it requires identifying specific joint dependencies. We employ a Taylor series expansion $\phi(z) = \sum_{k=0}^{\infty} \frac{\phi^{(k)}(0)}{k!} z^k$ to generalize our analysis to any smooth activation function $\phi$, allowing us to model learning dynamics for interaction terms of arbitrary order.

The following definition formalizes this structure to categorize learning difficulty.

\begin{defn}[Order-$r$ Interaction]
   \label{def:order_r_interaction}
   A target variable $y$ is driven by an Order-$r$ Interaction supported on a set of input vectors $\mathcal{S}$, if $y$ depends on the joint values of the vectors in $\mathcal{S}$ ($|\mathcal{S}|=r$), while remaining independent of any proper subset $\mathcal{S}' \subsetneq \mathcal{S}$. For any vector $\bx \in \mathcal{S}$, we say that $\bx$ is in the Order-$r$ Interaction of $y$.
\end{defn}

This definition captures the hierarchical nature of reasoning shown in Figure \ref{fig:parity_problem}. For instance, $x_{17}$ is an Order-2 Interaction of $\{\bx_2, \bx_4\}$. Deeper in the network, $x_{21}$ depends on the joint values of $\{\bx_2, \bx_4, \bx_6, \bx_8\}$, constituting an Order-4 Interaction relative to the atomic inputs. However, this high-order interaction is formed by aggregating lower-order components: $x_{21}$ exhibits an Order-2 dependency on the intermediate features $\{x_{17}, x_{18}\}$.

To solve the $k$-parity problem, the optimizer must distinguish relevant vectors (those in $\mathcal{S}$) from irrelevant ones via the attention weights. We analyze the gradient of the loss to see if it provides a directional signal pointing to the relevant input vectors.

\begin{thm}[Generalized Gradient Signal Decomposition]
   \label{thm:gradient_hierarchy_general}
   Consider the Transformer model defined in Section \ref{sec:transformer_model} with a smooth activation function $\phi$. 
   Let $\gamma_r = \phi^{(r)}(0)/(r-1)!$ be the interaction coefficient for order $r$.
   Let $\mathbb{I}_{j}^{(r)}$ indicate whether input $\bx_j$ is part of the Order-$r$ Interaction set $\mathcal{S}^{(r)}$. 
   Let $C_{\mathrm{signal}}^{(r)}$ denote the gradient signal of the Order-$r$ interaction term.
   With probability at least $1 - p$, the finite-sample noise $\kappa$ satisfies the bound $O\left(\sqrt{\frac{\log (m/p)}{n}}\right)$, and the gradient with respect to the attention weight $w_{j,m}$ follows the generalized asymptotic expansion:
   \begin{equation}
      \label{eq:grad_asymptotic_general}
      \begin{split}
          \frac{\partial \mathcal{L}}{\partial w_{j,m}} = 
          & \sum_{r=1}^{\infty} \gamma_r \cdot \Bigg[
          \underbrace{-\Theta(m^{-r}) \cdot \mathbb{I}_j^{(r)} \cdot C_{\mathrm{signal}}^{(r)} }_{\text{Order-}r \text{ Signal}} \\
          & + \underbrace{\Theta(m^{-(r+1)}) \cdot C_{\mathrm{signal}}^{(r)} }_{\text{Order-}r \text{ Bias}} 
          \Bigg]
          + \underbrace{O\left( \frac{\kappa}{m} \right)}_{\text{Noise}}.
      \end{split}
   \end{equation}
\end{thm}

\textbf{Mechanism of Gradient Alignment.}
Equation \ref{eq:grad_asymptotic_general} reveals that the total gradient is a superposition of signals from all interaction orders. The crucial mechanism for feature selection lies in the indicator term $\mathbb{I}_j^{(r)}$. For a specific input $\bx_j$, if it belongs to the support set of an Order-$r$ interaction (e.g., $\bx_j \in \mathcal{S}^{(r)}$), the gradient receives a consistent, structured "push" proportional to $\gamma_r C_{\mathrm{signal}}^{(r)}$. 
Conversely, for irrelevant inputs where $\forall r, \mathbb{I}_j^{(r)}=0$, the signal terms vanish, leaving only the bias and noise. This differentiation allows the attention mechanism to filter out noise and align its weights $w_{j,m}$ with the causal ancestors of the target, provided the signal strength is sufficient to survive the noise floor.

Based on this decomposition, we derive four critical scaling laws that govern the limits of learning:

\textbf{1. The Hierarchy of Learnability.}
The signal strength scales as $\Theta(m^{-r})$. This implies a rapid hierarchical decay: higher-order signals are significantly attenuated in the optimization landscape. For instance, an Order-4 signal is suppressed by a factor of $m^2$ relative to an Order-2 signal. Consequently, learning complex logic is not just combinatorially harder, but optimizationally slower, requiring significantly more iterations to accumulate gradient updates.

\textbf{2. Sample Complexity Scaling ($n \propto m^{2(r-1)}$).}
For the model to capture an Order-$r$ dependency, the effective signal $\Theta(m^{-r})$ must dominate the background noise $O(\kappa/m)$. This requires $\kappa \ll m^{-(r-1)}$. Substituting the noise bound $\kappa \sim n^{-1/2}$, we derive the critical sample size:
\begin{equation}
   n = \Omega\left(m^{2(r-1)} \log \frac{m}{p}\right).
\end{equation}
This indicates an explosive growth in data requirements. While Order-2 interactions require $n \propto m^2$, Order-4 interactions demand $n \propto m^6$.

\textbf{3. The Context Length Bottleneck.}
The context length $m$ acts as a fundamental bottleneck. As $m$ increases, it dilutes the useful signal polynomially ($m^{-r}$) while only suppressing noise linearly ($1/m$). This drastically reduces the Signal-to-Noise Ratio. In long-context scenarios, the gradient signal for high-order logic becomes indistinguishable from statistical fluctuations, making the retrieval of complex dependencies computationally prohibitive without massive data scaling.

\textbf{4. The Difficulty of Implicit Reasoning.}
This framework formally explains the fundamental challenge of implicit reasoning compared to explicit reasoning. In Explicit CoT, the model decomposes a target (e.g., $x_{21}$) into local steps (e.g., $x_{17}, x_{18}$), maintaining low-order interactions ($r=2$) with strong signals ($\Theta(m^{-2})$).
However, compressing these steps into a latent state forces the model to learn the dependency directly. If the compressed latent token fails to perfectly encode the intermediates, the effective interaction order increases (e.g., to $r=3$ or $r=4$). Theorem \ref{thm:gradient_hierarchy_general} dictates that this shift causes the learning signal to vanish exponentially ($\Theta(m^{-r})$) while the data requirement explodes polynomially ($m^{2(r-1)}$).

We further analyze how the specific implementation of latent tokens exacerbates this difficulty. Detailed theoretical derivations and rigorous justifications are provided in Appendix~\ref{sec:two_token_the}

\textbf{Imp.Base-1:} Since the latent token is an input-independent parameter, it is mathematically decoupled from the context. Consequently, it cannot form dynamic interaction terms of any order with the input tokens, nor can it inject additional reasoning information into the sequence. This renders the token effectively inert for bridging the reasoning gap.

\textbf{Imp.Base-2:} Unlike explicit CoT, this approach fails to reduce the task to low-order interactions. Although theoretically enabling a model limited to $l$-order terms to capture interactions beyond order $l$ by aggregating global context, this capability comes at the cost of injecting substantial noise into the optimization landscape, making the implicit reasoning path arduous to learn.

\subsection{Empirical Analysis: Semantic Shortcuts vs. Logical Irreducibility}
\label{sec:empirical_analysis}

Theorem \ref{thm:gradient_hierarchy_general} establishes that learning implies a trade-off: implicit reasoning reduces context length $m$ but potentially increases the interaction order $r$. This raises a critical empirical question: why do current methods significantly degrade performance in mathematical tasks, while often maintaining effectiveness in commonsense reasoning? To answer this, we must analyze the inherent interaction orders required by these respective tasks.

\subsubsection{Quantifying Interaction Orders}

We propose a data-driven approach to estimate the interaction order using Pointwise Mutual Information (PMI)\cite{bouma2009normalized}. A core prerequisite for identifying true high-order dependencies is Interaction Synergy ($\sigma$)—the excess information provided by a subset of input tokens $\mathcal{S}_k$ regarding the target $y$, beyond what is available from any of its proper sub-components:
\begin{equation}
   \sigma(X_{\mathcal{S}_k}; y) = \text{PMI}(X_{\mathcal{S}_k}; y) - \max_{\substack{\mathcal{S}' \subset \mathcal{S}_k \\ |\mathcal{S}'|=k-1}} \text{PMI}(X_{\mathcal{S}'}; y).
\end{equation}
A positive synergy ($\sigma > 0$) confirms that the relationship is irreducible, meaning it cannot be decomposed into lower-order correlations. Based on this, we define two metrics to capture the nature of the data:

\textbf{1. Effective Interaction Density ($\rho_k$):} This measures the frequency of valid irreducible shortcuts per sample:
\begin{equation}
\rho_k = \frac{1}{N} \sum_{i=1}^N \sum_{\mathcal{S}_k \subseteq \mathbf{x}_i} \mathbb{I}(\sigma(X_{\mathcal{S}_k}; y_i) > 0 \land \text{PMI} > \gamma).
\end{equation}

\textbf{2. Interaction Quality ($\phi_k$):} This quantifies the \textit{strength} of the high-order signal by averaging the synergy of valid interactions. High density combined with low quality suggests redundancy (many weak links), whereas high quality implies structural necessity (strong logical bonds).

\begin{figure}[t]
    \centering
    \includegraphics[width=1\linewidth]{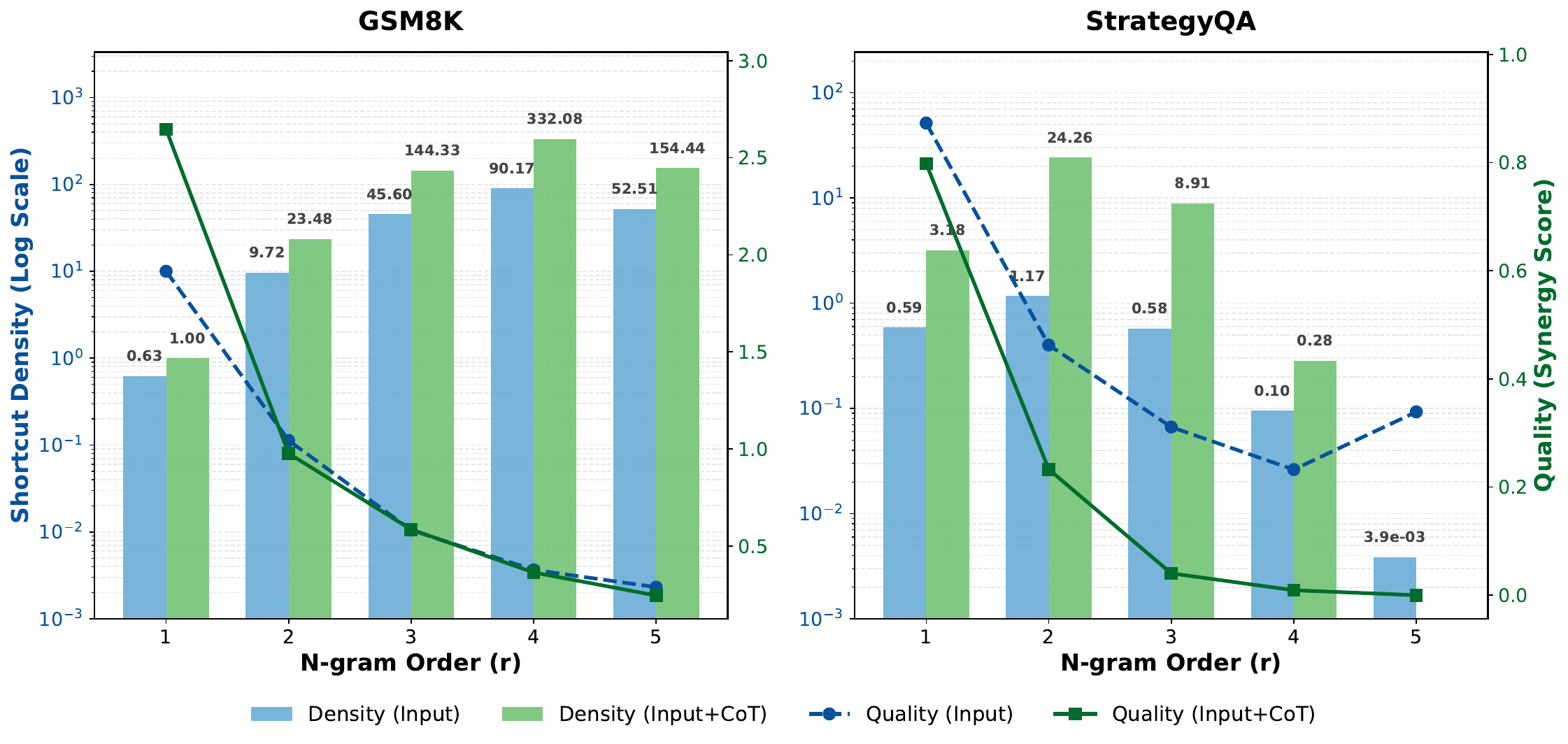}
    \caption{Interaction Landscape. We compare Effective Interaction Density ($\rho_k$, bars, left log-axis) and Interaction Quality ($\phi_k$, lines, right axis). Commonsense tasks (right) show low-quality high-order interactions, implying reducibility. Mathematical tasks (left) maintain high quality, indicating irreducible logical dependencies. Note that we sampled fewer candidate tokens for higher-order interactions to mitigate combinatorial explosion.}
    \label{fig:cot_impact}
\end{figure}

\subsubsection{Results and Analysis}

We apply these metrics to representative datasets (e.g., StrategyQA for commonsense \cite{geva2021did}, GSM8k for math \cite{cobbe2021training}). Figure \ref{fig:cot_impact} visualizes the interaction landscape, revealing a fundamental dichotomy that explains the varying effectiveness of Implicit CoT.

\textbf{Commonsense: The Prevalence of Semantic Shortcuts.}
For commonsense tasks, our analysis reveals that while high-order interactions exist in density ($\rho_k$), the distribution is dominated by low-order terms, indicating a strong low-order correlation between questions and answers. Upon introducing Chain-of-Thought (CoT), although the number of interaction terms increases significantly, the relative Interaction Quality at higher orders generally declines. This suggests that the effective new reasoning information provided by CoT is primarily concentrated in low-order terms, aligning with the "Arthur $\to$ Europe" scenario in Figure \ref{fig:cot_comparison} where the reasoning possesses a Reducible Structure. Although the explicit chain involves causal steps (Arthur $\to \dots \to$ Europe), the reasoning is primarily chain-like; thus, bypassing intermediate steps to map input directly to output does not cause an explosion in interaction order ($r \sim O(1)$).
In this scenario, Implicit CoT compresses the context length $m$ without incurring the penalty of increasing $r$. According to Eq. \eqref{eq:grad_asymptotic_general}, reducing $m$ while keeping $r$ constant significantly boosts the gradient signal ($\Theta(m^{-r})$) and reduces noise ($O(1/m)$), making Implicit CoT more effective.

\textbf{Mathematics: The Barrier of Logical Irreducibility.}
Conversely, mathematical reasoning features a greater prevalence of interaction terms at higher orders. Crucially, the introduction of CoT does not significantly reduce Interaction Quality across orders, indicating that the high-order information provided by CoT is non-redundant. Furthermore, CoT exhibits higher quality in low-order terms compared to raw input, demonstrating that CoT effectively lowers problem-solving difficulty, consistent with the Logical Irreducibility shown in Figure \ref{fig:cot_comparison}.
When solving such problems, explicit CoT decomposes the deep nested function into sequential, manageable steps (maintaining local $r=2$). However, forcing implicit reasoning to solve the problem directly in a single step causes the interaction order to scale with computational depth ($r \propto \text{Depth}$). As derived in Theorem \ref{thm:gradient_hierarchy_general}, this shift triggers an exponential decay in signal strength ($\Theta(m^{-r})$) that far outweighs the linear benefit of a shorter context. Consequently, the distinct high-order signal is buried beneath the noise, leading to model degradation.

\textit{Remark on Data Scarcity.} It is worth noting that even in math datasets, the presence of some low-order correlations indicates that restricted sample sizes often allow models to overfit to spurious shortcuts rather than learning true reasoning. This highlights the need for a rigorous testbed.

\begin{table}[t]
\centering
\caption{\textbf{Statistics of the NatBool-DAG Dataset.} The dataset is stratified by the number of reasoning hops (3 to 10). Distinct sets are generated for training, validation, and testing to ensure statistical significance across varying difficulty levels.}
\label{tab:natbool_stats}
\resizebox{0.8\linewidth}{!}{
\begin{tabular}{c|ccc}
\toprule
\textbf{Reasoning Hops} & \textbf{Train Size} & \textbf{Val Size} & \textbf{Test Size} \\
\midrule
3 Hops  & 469 & 187 & 220 \\
4 Hops  & 662 & 298 & 299 \\
5 Hops  & 911 & 204 & 276 \\
6 Hops  & 1,104 & 189 & 285 \\
7 Hops  & 1,349 & 235 & 287 \\
8 Hops  & 1,522 & 253 & 189 \\
9 Hops  & 1,712 & 256 & 229 \\
10 Hops & 1,886 & 227 & 177 \\
\bottomrule
\end{tabular}
}
\end{table}

\section{NatBool-DAG: A Benchmark for Irreducible Reasoning}
\label{sec:natbool_dag}

Section \ref{sec:empirical_analysis} reveals a critical flaw: prevalent low-order interactions in standard reasoning benchmarks often permit semantic shortcuts. Consequently, the success of existing methods may reflect statistical pattern matching rather than genuine high-order compression. To rigorously evaluate Implicit CoT, we introduce Natural-Language Boolean Directed Acyclic Graphs (NatBool-DAG), a synthetic benchmark ensuring \textit{Logical Irreducibility} and \textit{Sample Sufficiency} to enforce robust reasoning.

\subsection{Dataset Construction}

NatBool-DAG embeds random Directed Acyclic Graphs of Boolean operators (AND, OR, XOR, NOT) into coherent narratives, ensuring solutions rely strictly on internal logic rather than external knowledge. The construction pipeline involves four key steps:

\textbf{1. Logic Graph Generation.} We generate computational graphs where nodes denote binary states and edges represent logical dependencies. Topologies are randomized per sample to prevent structural overfitting.

\textbf{2. Natural Language Conversion.} Nodes are mapped to randomized entities (e.g., ``Alice'') to eliminate prior knowledge biases (e.g., semantic shortcuts), while logical gates are wrapped in narrative templates to ensure self-containment.

\textbf{3. Controllable Complexity.} Datasets are generated with variable reasoning depths ranging from 3 to 10 hops to strictly evaluate scaling behavior across complexity levels.

\textbf{4. Ground Truth CoT.} A rigorous step-by-step derivation is generated for each sample. 

Using this pipeline, we constructed the evaluation suite detailed in Table \ref{tab:natbool_stats}. While the validation and test sets maintain a relatively stable size, the training set expands as reasoning hops increase to accommodate the growing complexity. Examples are provided in Appendix \ref{sec:data_example}.

\section{Methodology: Implicit Reasoning via Latent State Alignment}

\subsection{Theoretical Motivation}
To mitigate the exponential complexity spike in implicit CoT (Theorem \ref{thm:gradient_hierarchy_general}), latent tokens $\bc_s$ must effectively substitute omitted steps as low-order computational supports. We posit that if $\bc_s$ aligns with the \textit{distribution} of underlying logical states, the model can learn manageable low-order transitions rather than reconstructing high-order logic from inputs. Thus, efficient compression hinges on aligning the latent space with the explicit reasoning process.

\subsection{ALiCoT: Aligned Implicit Chain-of-Thought}
We propose \textbf{ALiCoT} to enforce this alignment by ensuring latent tokens remain distributionally distinguishable regarding ground-truth steps. Formally, utilizing a sequence of independent learnable parameters as latent tokens, we introduce an auxiliary classifier $g(\cdot)$ mapping the latent state $h_{\text{latent}}$ back to the discrete logical step $z$. The objective is augmented as:
\begin{equation}
    \mathcal{L} = \mathcal{L}_{\text{task}} + \lambda \mathcal{L}_{\text{align}} = \mathcal{L}_{\text{task}} - \lambda \sum_{i} \log P(z_i | g(h_{\text{latent}}))
    \label{eq:alignment_loss}
\end{equation}
Minimizing $\mathcal{L}_{\text{align}}$ preserves the discriminative power of explicit steps within the latent space, thereby preventing interaction order explosion.

\begin{figure}[t]
    \centering
    \includegraphics[width=1\linewidth]{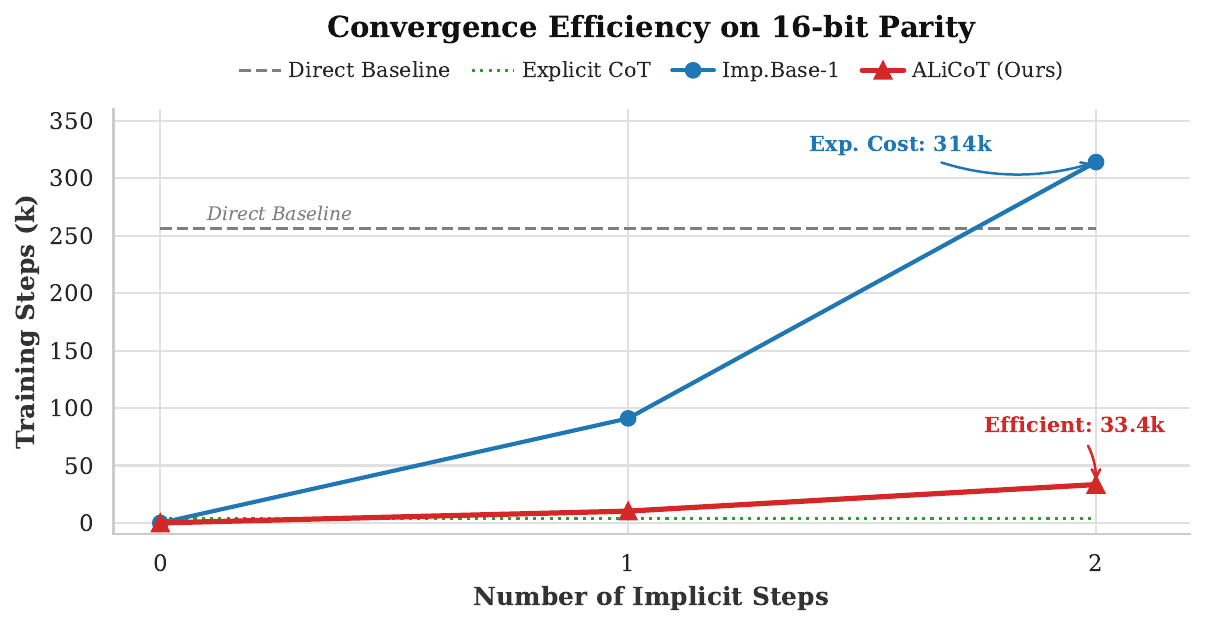} 
    \caption{\textbf{Convergence Efficiency on 16-bit Parity.} We plot the training steps required to reach $100\%$ accuracy. Imp.Base-1 suffers exponential growth, while ALiCoT maintains a flat curve.}
    \label{fig:parity_results}
\end{figure}
\subsection{Preliminary Verification on Parity Problem}

We validate ALiCoT on the 16-bit Parity problem to assess learning efficiency under compression. Specifically, we track the number of training steps required to achieve 100\% accuracy as the number of implicit steps increases. As shown in Figure \ref{fig:parity_results}, the standard iCoT baseline exhibits an exponential delay in convergence (surging to 314k steps), empirically confirming the hardness of learning implicit reasoning (Theorem \ref{thm:gradient_hierarchy_general}). In contrast, ALiCoT maintains a flat convergence curve (peaking at only 33.4k), demonstrating that our alignment objective effectively prevents the explosion in training cost.

\subsection{Adaptation to Natural Language}

While the discrete classification approach proves effective for bounded state spaces (e.g., the Parity problem), it is intractable for natural language reasoning due to the combinatorial and open-ended nature of intermediate thoughts. Enumerating all potential states to construct a discriminator is infeasible.

To address this, we adapt \textbf{ALiCoT} by leveraging the high-dimensional hidden states of the Transformer as continuous semantic targets. Instead of discrete labels, we treat the explicit Chain-of-Thought representations as prototype vectors. Consequently, we reformulate the alignment objective in Eq.~\eqref{eq:alignment_loss} from discrete cross-entropy to a feature-level cosine distance metric. By maximizing the cosine similarity between the latent token $\mathbf{c}_s$ and the representation of the $s$-th CoT step $\mathbf{t}_s$, we guide the implicit reasoning to align with the correct semantic direction:

\begin{equation}
\mathcal{L}_{\text{align}} = 1 - \cos(\mathbf{c}_s, \text{SG}(\mathbf{t}_s)) = 1 - \frac{\mathbf{c}_s \cdot \text{SG}(\mathbf{t}_s)}{|\mathbf{c}_s|_2 |\text{SG}(\mathbf{t}_s)|_2}
\label{eq:nl_alignment_loss}
\end{equation}
\vspace{-1mm}
where $\text{SG}(\cdot)$ denotes the stop-gradient operation. This objective acts as a proxy for a Softmax classifier, ensuring the latent representation captures the necessary intermediate information without requiring discrete enumeration.

\begin{figure}[t]
    \centering
    \includegraphics[width=0.8\linewidth]{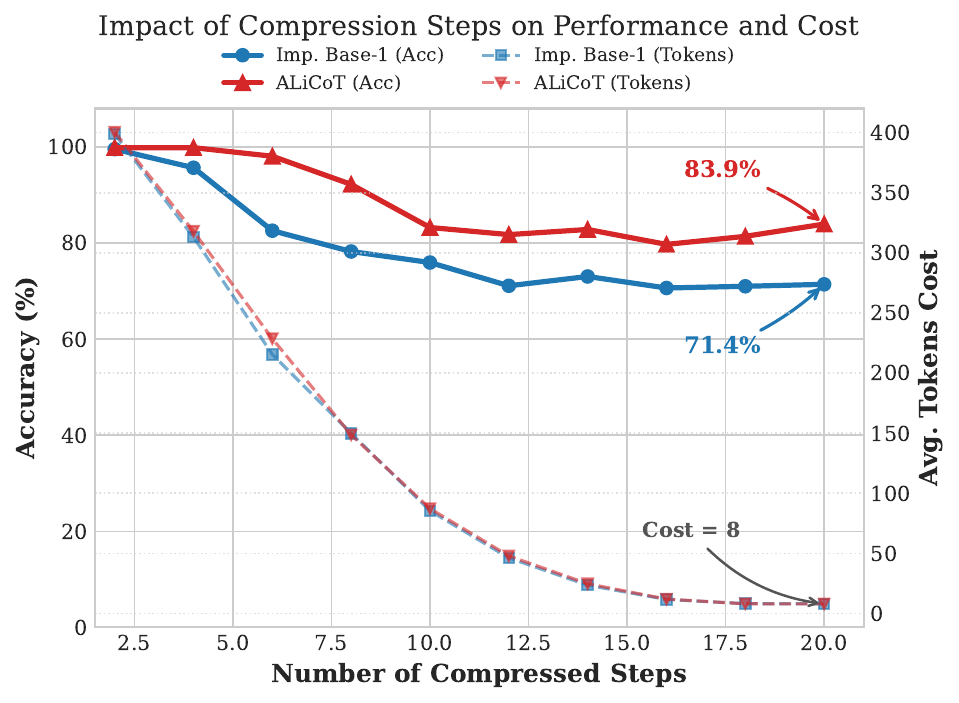}
    \caption{\textbf{Impact of compression steps on model performance and computational cost.} 
    Solid lines represent accuracy (left axis), while dashed lines indicate the average token cost (right axis). 
    As the number of compressed steps increases, \textbf{ALiCoT} (red) demonstrates superior robustness compared to Imp.~Base-1 (blue). Notably, at the highest compression level (Step 20), both methods incur a minimal cost of $\sim$8 tokens, yet ALiCoT outperforms the baseline by $>12\%$ (83.94\% vs. 71.40\%).}
\vspace{-4mm}
    \label{fig:compression_impact}
\end{figure}

\textbf{Experimental Results.} We conduct our evaluations using Qwen3-0.6B \cite{yang2025qwen3} as the backbone model. To ensure a fair comparison, all experiments adhere to identical hyperparameter configurations (detailed in Appendix~\ref{sec:experiments_details}). As presented in Figure~\ref{fig:compression_impact}, we analyze the trade-off between performance and efficiency. While token consumption for both methods drops significantly as compression steps increase and converges to approximately 8 tokens at the maximum level, their performance trajectories diverge markedly. The standard implicit baseline (Imp.~Base-1) struggles to retain reasoning capabilities, with accuracy degrading to 71.40\%. In contrast, \textbf{ALiCoT} demonstrates superior robustness, maintaining an accuracy of 83.94\% under the same extreme compression constraint. This performance gap of over 12\% confirms that feature-level alignment successfully distills reasoning structures into the latent space, enabling high performance even with minimal computational overhead. Comprehensive experimental results are provided in Appendix~\ref{sec:main_experimental_results}.

\noindent \textbf{Dual Advantage of Speed and Accuracy.}
Table \ref{tab:qwen_speed_acc} highlights the superior trade-off achieved by ALiCoT. In terms of efficiency, our method matches the aggressive \textbf{54.40$\times$ speedup} of the static implicit baseline (Imp.~Base-1). However, typical implicit methods often suffer from severe performance degradation to gain this speed. As shown in the 4B model results, while the uncompressed explicit CoT achieves a ceiling of 100\% accuracy, the baseline degrades significantly to 77.88\%. In contrast, ALiCoT successfully mitigates this loss, recovering \textbf{95.01\%} of the full reasoning capability. This demonstrates that ALiCoT achieves an optimal efficiency-accuracy trade-off: it delivers near-lossless reasoning performance comparable to the computationally expensive CoT, while strictly maintaining extremely fast inference speeds.

\begin{table}[t]
    \centering
    \small
    \setlength{\tabcolsep}{2pt}
    \renewcommand{\arraystretch}{1.1}
    \caption{\textbf{Performance and Efficiency on Qwen3.} ALiCoT matches the massive speedup of explicit CoT ($54.4\times$) while ensuring high accuracy. It reaches \textbf{95.01\%} on the 4B model, significantly surpassing Imp.Baselines without any latency penalty. Note that due to computational constraints, Imp. Base-2 is evaluated only on Qwen3-0.6B. Additional results on Llama and other model scales are detailed in Appendix~\ref{sec:main_experimental_results}.}
    \label{tab:qwen_speed_acc}
    
    \begin{tabular}{l l c c c c}
        \toprule
        \multirow{2}{*}{\textbf{Model}} & \multirow{2}{*}{\textbf{Method}} & \multicolumn{2}{c}{\textbf{Acc (\%)}} & \textbf{Avg.} & \textbf{Speed} \\
        \cmidrule(lr){3-4}
         & & \textbf{Hop3} & \textbf{Hop10} & \textbf{Acc} & \textbf{Up} \\
        \midrule
        
        \multirow{3}{*}{\textbf{Qwen3 (0.6B)}} 
          & Imp. Base-1 & 68.6 & 70.6 & 71.41 & $54.40\times$ \\
          & Imp. Base-2 & 33.5 & 53.8 & 45.87 & $7.65\times$ \\
          & \textbf{ALiCoT} & \textbf{87.3} & \textbf{83.6} & \textbf{83.94} & $\mathbf{54.40\times}$ \\
        \midrule
        
        \multirow{2}{*}{\textbf{Qwen3 (4B)}} 
          & Imp. Base-1 & 79.1 & 74.0 & 77.88 & $54.40\times$ \\
          & \textbf{ALiCoT} & \textbf{100.0} & \textbf{92.7} & \textbf{95.01} & $\mathbf{54.40\times}$ \\
        
        \bottomrule
    \end{tabular}
\end{table}
\vspace{-2mm}
\section{Discussion and Conclusion}
\vspace{-2mm}

This study establishes the theoretical complexity of CoT compression, identifying interaction order explosion as the fundamental barrier. We reveal that omitting intermediate steps causes learning signals to decay exponentially. We distill these findings into the Law of CoT Compression, and quantifying the CoT Compression Complexity ($\Phi$) as:
$\Phi \propto e^{\lambda \cdot c \cdot \rho}$
where $c$ denotes the number of compressed steps and $\rho$ represents the Logical Density (irreducibility) of the task. This formulation explains the sharp contrast observed in our experiments: while low-density commonsense tasks (low $\rho$) remain compressible, high-rigidity reasoning triggers an exponential cost explosion.
To rigorously validate this and address the "semantic shortcuts" in existing datasets, we introduce NatBool-DAG, a benchmark ensuring logical irreducibility. Guided by these insights, we propose ALiCoT, which effectively minimizes the alignment coefficient $\lambda$ by aligning latent tokens with explicit reasoning steps. This renders the exponential complexity manageable, allowing models to successfully internalize complex logic while maintaining high efficiency. Collectively, these contributions provide both the theoretical rigor and practical methodology necessary to achieve genuinely efficient CoT compression.

Future work will investigate the impact of model scaling on $\lambda$ and explore advanced alignment strategies (e.g., contrastive learning) to further minimize this coefficient. Ultimately, overcoming the fundamental exponential growth of $\Phi$ to enable extremely deep implicit reasoning remains a critical long-term objective.


\newpage
\appendix
\onecolumn
\section{Proof of Theorem \ref{thm:gradient_hierarchy_general}}
\label{sec:final_proof}

\subsection{Transformer Model}
\label{sec:transformer_model}

We discuss the learning ability of models in constructing implicit reasoning chains under a one-layer Transformer architecture. To simplify the analysis, we use absolute position encoding and single-head Softmax attention, and the single-layer Transformer also omits residual connections.

\textbf{Data encoding}: We utilize high-dimensional vector representations initialized via a uniform distribution $\mathcal{U}(-1, 1)$. To preserve distributional symmetry and geometric separability for the binary input tokens, we adopt a coupled initialization strategy. Let $\mathbf{v} \in \mathbb{R}^{d_{model}}$ be the base embedding vector corresponding to the token "$1$", with its elements sampled independently from $\mathcal{U}(-1, 1)$. Accordingly, the token "$-1$" is encoded as the negation $-\mathbf{v}$. Formally, for the input token $x_j \in \{1, -1\}$ at the $j$-th position, the resulting encoding $\mathbf{e}_j$ is defined as:
\begin{equation}
   \mathbf{e}_j = x_j \mathbf{v}.
\end{equation}
This formulation ensures that the embedding direction is strictly determined by the token sign while sharing the same magnitude and global statistical properties.

\textbf{Softmax attention layer}: The attention layer consists of key, query, and value matrices $\bK, \bQ, \bV$. To make the dynamical analysis tractable, following \citet{Zhang23, Huang23, Mahankali23, Kim24}, we focus on position encoding when obtaining attention scores, while the value matrix operates solely on the semantic embeddings $\be$. Therefore, in this paper, we treat the position encoding $\bp$ separately from the input embedding $\be$. The forward propagation of the attention layer is formally defined as:

\begin{equation}
   \label{eqn:attention}
   \hat{\bz}_m = \sum_{j=1}^{m-1} \bV\be_j \cdot \sm(\hat{\bp}_j^\top \bK^\top\bQ \hat{\bp}_m) = \sum_{j=1}^{m-1} \sigma_j(\bw_m) \be_j,
\end{equation}
where $\sigma_j(\bw_m)$ represents the attention weights (softmax scores) allocated to the $j$-th token. Specifically, we aggregate the key and query projections into a single learnable interaction matrix $\bW \triangleq \bK^\top \bQ$. Consequently, the attention weights are computed as $\sigma_j(\bw_m) = \frac{\exp(w_{j,m})}{\sum_{\alpha=1}^{m-1} \exp(w_{\alpha,m})}$, where the scalar logit $w_{j,m} = \hat{\bp}_j^\top \bW \hat{\bp}_m$ quantifies the position-based compatibility between the query at step $m$ and the key at step $j$.

\textbf{Feedforward layer}: To fully adapt to the parity problem, following \citet{ICLR2025_6e2986de}, we use a mapping function $\phi:[-1,1]\to [-1,1]$, requiring $\phi(0)=-1$, $\phi(\pm 1)=1$, and set the function to be smooth and differentiable. Finally, we define $\phi(t) = 1-ct^2+ct^4$ and $\phi'(t) = -2ct+4ct^3$.

\textbf{Output probability}: To perform the classification task, the hidden state $\bh_m$ is projected into the logit space. We employ the Softmax function to obtain the predicted probability distribution vector $\hat{\by}_m \in \mathbb{R}^2$:
\begin{equation}
   \label{eqn:output}
   \hat{\by}_m = \text{Softmax}(\bW_o^\top \bh_m), \quad \text{where } [\hat{\by}_m]_k = \frac{\exp(\mathbf{c}_k^\top \bh_m)}{\sum_{j=1}^{2} \exp(\mathbf{c}_j^\top \bh_m)}.
\end{equation}

Here, $\bW_o \in \mathbb{R}^{ d_{model} \times 2} = \left[  \mathbf{c}_1, \mathbf{c}_2 \right]$ denotes the weights of the output projection head, and $[\hat{\by}_m]_k$ represents the probability of the $k$-th class. Analogous to our input encoding strategy, we enforce a symmetric structure on the output weights. We initialize a vector $\mathbf{u} \in \mathbb{R}^{d_{model}}$ such that the weight for the class "$1$" is defined as $\mathbf{c}_1 = \mathbf{u}$, while the weight for the class "$-1$" is set as the negation, $\mathbf{c}_2 = -\mathbf{u}$. This construction ensures that the decision boundaries for the binary states remain geometrically symmetric.

\textbf{Loss function}: We define the optimization objective using the standard Cross-Entropy (CE) loss, rather than the Mean Squared Error (MSE). Given the ground-truth label $y \in \{1, -1\}$, the loss for the $m$-th token is the negative log-likelihood of the true class. Considering the binary output dimension, the loss is formulated as:
\begin{equation}
   \label{eqn:loss}
   \mathcal{L} = - \sum_{k=1}^{2} \mathbb{I}(y = l_k) \log([\hat{\by}_m]_k),
\end{equation}
where $\mathbb{I}(\cdot)$ is the indicator function. The term $l_k$ denotes the class label corresponding to the $k$-th index of the probability vector, specifically $l_1 = 1$ and $l_2 = -1$. Minimizing $\mathcal{L}$ is equivalent to maximizing the probability assigned to the correct label required for the reasoning step.

Thus, the transformer computes $\textstyle \TF(\bx_1,\cdots,\bx_{d+k-1}; \bW) = (\hat{\bx}_1,\cdots,\hat{\bx}_{d+k-1})$ where the original data $\hat{\bx}_j=\bx_j, j\in [d]$ remain unchanged and tokens $\hat{\bx}_{d+1},\cdots,\hat{\bx}_{d+k-1}$ are computed as $\hat{\bx}_m = \phi(\hat{\bz}_m)$.

\subsection{Parity-Invariant Correlation Analysis}
\label{sec:parity_invariant_correlation}

To analyze the interaction mechanism between input variables and the prediction target, we introduce a geometric interpretation based on the Hadamard product of their representations.

\textbf{Embedding Representations.}
We represent an input variable $x \in \{1, -1\}$ as a vector $x \bv$, and a target variable $y \in \{1, -1\}$ as $y \bu$, where $\bv, \bu \in \mathbb{R}^{d_{model}}$ are base vectors for the input and output spaces, respectively.

\textbf{Second-order Interaction.}
Consider a reasoning step where the target $y_m$ is determined by the parity of two input variables $x_i$ and $x_j$, i.e., $y_m = x_i x_j$. In the binary domain $\{1, -1\}$, this multiplicative relationship enforces the strict constraint:
\begin{equation}
    y_m \cdot x_i \cdot x_j = 1.
\end{equation}
This implies that the sign of the output must align with the combined sign of the inputs. To capture this geometric alignment, we define the interaction vector $\boldsymbol{\psi}$ as the element-wise product (Hadamard product, denoted by $\odot$) of the embeddings of these three variables:
\begin{equation}
\begin{aligned}
    \boldsymbol{\psi} &= (y_m \bu) \odot (x_i \bv) \odot (x_j \bv) \\
    &= (y_m \cdot x_i \cdot x_j) (\bu \odot \bv \odot \bv).
\end{aligned}
\end{equation}
Given the constraint $y_m x_i x_j \equiv 1$, the scalar coefficients cancel out, causing the interaction vector to collapse into a constant, input-independent state:
\begin{equation}
    \boldsymbol{\psi} = \bu \odot \bv^{\odot 2},
\end{equation}
where $\bv^{\odot 2} = \bv \odot \bv$ denotes the element-wise square. We term this property \textbf{Parity-Invariant Correlation}. It demonstrates that for any valid reasoning triplet $(x_i, x_j, y_m)$, their joint representation in the hyperspace always points to a fixed invariant direction $\bu \odot \bv^{\odot 2}$, regardless of the individual values of the variables.

\textbf{Generalization to High-order Interactions.}
This invariance property generalizes to $k$-th order interactions. Suppose a target $y_m$ is determined by the cumulative product of $k$ variables $\{x_{i_1}, \dots, x_{i_k}\}$, satisfying $y_m = \prod_{n=1}^k x_{i_n}$. The consistency condition becomes $y_m \cdot \prod_{n=1}^k x_{i_n} = 1$. Consequently, the cumulative Hadamard product of their vector representations remains invariant:
\begin{equation}
    \label{eqn:consistency_general}
    (y_m \bu) \odot \left( \bigodot_{n=1}^{k} (x_{i_n} \bv) \right) = \bu \odot \bv^{\odot k}.
\end{equation}
This result suggests that learning the underlying logical rule is geometrically equivalent to aligning the model's aggregate representation with the specific invariant direction $\bu \odot \bv^{\odot k}$.

\textbf{Support Set and Interaction Categories.}
To rigorously analyze the gradient components, we categorize the interaction terms based on their ability to reconstruct the target parity $y_m$. We consider a tuple of indices $J = (j_1, \dots, j_r)$ with $1 \le j_k \le m-1$. The collective contribution of these features is determined by the product $X_J = \prod_{k=1}^r x_{j_k}$.

Consistent with the parity constraint, we define a tuple $J$ as \textit{target-aligned} (or relevant) if the product of its components strictly recovers the target label (i.e., their correlation is trivial):
\begin{equation}
    \mathcal{S}_{r,m} := \left\{ (j_1, \dots, j_r) \in [m-1]^r \;\middle|\; y_m \cdot \prod_{k=1}^r x_{j_k} \equiv 1 \right\}.
\end{equation}
Tuples in $\mathcal{S}_{r,m}$ correspond to the true support of the underlying function (potentially including redundant pairs that cancel out).

Conversely, we define the set of \textit{misaligned} (or irrelevant) tuples $\mathcal{I}_{r,m}$ as those that fail to reconstruct the target:
\begin{equation}
    \mathcal{I}_{r,m} := \left\{ (j_1, \dots, j_r) \in [m-1]^r \;\middle|\; y_m \cdot \prod_{k=1}^r x_{j_k} \not\equiv 1 \right\}.
\end{equation}

Based on the rigorous definitions of $\mathcal{S}_{r,m}$ and $\mathcal{I}_{r,m}$, we examine the statistical behavior of the scalar contraction over the dataset. This analysis reveals a fundamental dichotomy in how these terms contribute to the gradient updates:

\textbf{Statistical Distinction.}
We consider the average contribution of a specific index tuple $J = (j_1, \dots, j_r)$ across $n$ samples:

\begin{itemize}
    \item \textbf{Relevant Terms (Deterministic Signal):}
    For any tuple $J \in \mathcal{S}_{r,m}$, the parity constraint ensures that $y_m \cdot \prod_{k=1}^r x_{j_k} \equiv 1$. Consequently, the contraction collapses into a deterministic constant for every sample, yielding a stable average:
    \begin{equation}
        \frac{1}{n} \sum_{i=1}^n \tc{y_m^{(i)} \bu, x_{j_1}^{(i)} \bv, \dots, x_{j_r}^{(i)} \bv} = C_{\text{signal}}.
    \end{equation}

    \item \textbf{Irrelevant Terms (Zero-Mean Noise):}
    In contrast, for any tuple $J \in \mathcal{I}_{r,m}$, the product term represents a non-trivial parity function. Under uniform inputs, the single-sample contraction is stochastic:
    \begin{equation}
        \tc{y_m \bu, x_{j_1} \bv, \dots, x_{j_r} \bv} \sim \text{Uniform}(\{+C_{\text{signal}}, -C_{\text{signal}}\}).
    \end{equation}
    Since these terms are uniformly distributed around zero, their average over the dataset is a sum of zero-mean random variables, which tends to vanish as $n$ increases.
\end{itemize}

\textbf{Concentration of Irrelevant Interactions.}
To ensure the model learns the correct rule, the signal $C_{\text{signal}}$ must dominate the worst-case noise from irrelevant terms. We formalize this bound using the concentration of measure.

\begin{lemma}[Concentration of Irrelevant Terms]\label{thm:kappa}
Let $\mathcal{I}_{r,m}$ be the set of misaligned index tuples. If inputs are i.i.d. uniform over $\{\pm 1\}$, then for any confidence level $1-p$, the maximum magnitude of the noise interactions is bounded by:
\begin{equation}
    \max_{J \in \mathcal{I}_{r,m}} \left| \frac{1}{n} \sum_{i=1}^n \tc{y_m^{(i)} \bu, x_{j_1}^{(i)} \bv, \dots, x_{j_r}^{(i)} \bv} \right| \le \kappa,
\end{equation}
where $\kappa := C_{\text{signal}} \cdot \sqrt{\frac{2}{n} \log \frac{2|\mathcal{I}_{r,m}|}{p}}$.
\end{lemma}

\begin{proof}
The proof relies on the zero-mean property of terms in $\mathcal{I}_{r,m}$ and a union bound argument:

1.  \textbf{Point-wise Bound via Hoeffding:}
    Fix an irrelevant tuple $J \in \mathcal{I}_{r,m}$. The summation consists of $n$ independent random variables bounded by $[-C_{\text{signal}}, C_{\text{signal}}]$ with zero expectation. By Hoeffding's inequality:
    \begin{equation}
        \Pr\left( \left| \frac{1}{n} \sum_{i=1}^n \tc{y_m^{(i)} \bu, x_{j_1}^{(i)} \bv, \dots, x_{j_r}^{(i)} \bv} \right| \ge \kappa \right) \le 2 \exp\left( -\frac{n\kappa^2}{2 C_{\text{signal}}^2} \right).
    \end{equation}

2.  \textbf{Uniform Bound via Union Bounding:}
    We must ensure that \textit{none} of the potential irrelevant terms masquerade as a signal. Applying the union bound over all tuples in $\mathcal{I}_{r,m}$, we limit the probability that the \textit{maximum} noise exceeds $\kappa$:
    \begin{equation}
        \Pr\left( \max_{J \in \mathcal{I}_{r,m}} \left| \frac{1}{n} \sum_{i=1}^n \tc{y_m^{(i)} \bu, x_{j_1}^{(i)} \bv, \dots, x_{j_r}^{(i)} \bv} \right| \ge \kappa \right) \le |\mathcal{I}_{r,m}| \cdot 2 \exp\left( -\frac{n\kappa^2}{2 C_{\text{signal}}^2} \right).
    \end{equation}
    Setting the right-hand side to $p$ and solving for $\kappa$ yields the stated bound.
\end{proof}

\textbf{Implication.} This result guarantees that with sufficient samples ($n = \tilde{\Omega}(\log m)$), there exists a clear separation margin: the signal $C_{\text{signal}}$ dominates the noise floor $\kappa$, driving the gradient descent towards the correct support $\mathcal{S}_{r,m}$.

\subsection{Gradient Derivation for Attention Weights}

To analyze the learning dynamics, particularly for tasks requiring high-frequency representations like parity, we derive the gradient of the loss $\mathcal{L}$ with respect to the attention weights $w_{j,m}$. For the sake of clarity in the initial derivation, we define $\mathcal{L}$ as the standard Cross-Entropy loss for a \textbf{single training instance} with target class $t$, denoted as $\mathcal{L} = -\log([\hat{\by}_m]_t)$. The generalization to the mean loss over a batch of $n$ samples will be incorporated subsequent to the derivation of the per-sample gradient components.

Applying the chain rule, the gradient can be decomposed into four primary components:
\begin{equation}
   \label{eq:gradient_chain_rule}
   \frac{\partial \mathcal{L}}{\partial w_{j,m}} = \underbrace{\frac{\partial \mathcal{L}}{\partial \hat{\by}_m} \cdot \frac{\partial \hat{\by}_m}{\partial \bh_m}}_{\text{(I)}} \cdot \underbrace{\frac{\partial \bh_m}{\partial \hat{\bz}_m}}_{\text{(II)}} \cdot \underbrace{\frac{\partial \hat{\bz}_m}{\partial w_{j,m}}}_{\text{(III)}}.
\end{equation}

\paragraph{Part (I): Loss and Output Projection.}
We derive the gradient of the loss with respect to the hidden state $\bh_m$ by decomposing the chain rule through the logits $\mathbf{z} = \bW_o^\top \bh_m$. Let $z_i = \mathbf{c}_i^\top \bh_m$ be the $i$-th logit. The gradient can be expressed as:
\begin{equation}
    \frac{\partial \mathcal{L}}{\partial \bh_m} = \sum_{i=1}^K \frac{\partial \mathcal{L}}{\partial z_i} \frac{\partial z_i}{\partial \bh_m}.
\end{equation}
First, we compute the derivative of the Cross-Entropy loss $\mathcal{L} = -\log(p_t)$ with respect to the logit $z_i$, where $p_k = [\hat{\by}_m]_k = \text{softmax}(\mathbf{z})_k$ denotes the predicted probability for class $k$, and $c$ is the index of the true class. Applying the chain rule, we have:
\begin{equation}
\begin{aligned}
    \frac{\partial \mathcal{L}}{\partial z_i} &= \sum_{k=1}^K \frac{\partial \mathcal{L}}{\partial p_k} \frac{\partial p_k}{\partial z_i}.
\end{aligned}
\end{equation}
Since the loss only depends on the target class probability $p_t$, the term $\frac{\partial \mathcal{L}}{\partial p_k}$ is non-zero only when $k=c$, specifically $\frac{\partial \mathcal{L}}{\partial p_t} = -\frac{1}{p_t}$. Using the derivative of the Softmax function $\frac{\partial p_k}{\partial z_i} = p_k (\delta_{ki} - p_i)$, the expression simplifies to:
\begin{equation}
\begin{aligned}
    \frac{\partial \mathcal{L}}{\partial z_i} &= -\frac{1}{p_t} \cdot \frac{\partial p_t}{\partial z_i} = -\frac{1}{p_t} \cdot p_t (\delta_{ci} - p_i) = p_i - \delta_{ci}.
\end{aligned}
\end{equation}
In vector notation, this corresponds to the residual vector $\frac{\partial \mathcal{L}}{\partial \mathbf{z}} = (\hat{\by}_m - \bm{1}_t)^\top$, where $\bm{1}_t$ is the one-hot encoding of the target class.
Next, considering the linear projection $z_i = \mathbf{c}_i^\top \bh_m$, the partial derivative is simply the class embedding: $\frac{\partial z_i}{\partial \bh_m} = \mathbf{c}_i^\top$. Substituting these back into the summation yields:
\begin{equation}
\begin{aligned}
    \frac{\partial \mathcal{L}}{\partial \bh_m} &= \sum_{i=1}^K (p_i - \delta_{ci}) \mathbf{c}_i^\top = \left( \sum_{i=1}^K p_i \mathbf{c}_i - \mathbf{c}_t \right)^\top = -\left( \mathbf{c}_t - \mathbb{E}_{\hat{\by}_m}[\mathbf{c}] \right)^\top.
\end{aligned}
\end{equation}
This term represents the error signal in the embedding space, pointing from the current weighted expected embedding $\mathbb{E}_{\hat{\by}_m}[\mathbf{c}]$ towards the target class centroid $\mathbf{c}_t$.

\paragraph{Part (II): Feed-Forward Network.}
The gradient propagates through the element-wise Feed-Forward layer, defined as $\bh_m = \phi(\hat{\bz}_m)$. Since $\phi$ is applied independently to each dimension of the input vector $\hat{\bz}_m$ without any learnable weight matrices, the Jacobian matrix is diagonal. Specifically:
\begin{equation}
    \frac{\partial \bh_m}{\partial \hat{\bz}_m} = \text{diag}(\phi'(\hat{\bz}_m)),
\end{equation}
where $\phi'(t) = -2ct+4ct^3$ is the element-wise derivative of the polynomial mapping function. This term acts as a gate, scaling the gradient flow based on the magnitude of the attention output features.

\paragraph{Part (III): Attention Mechanism.}
Finally, we compute the derivative of the attention output $\hat{\bz}_m = \sum_{\alpha=1}^{m-1} \sigma_\alpha(\bw_m) \bx_\alpha$ with respect to the attention logits $w_{j,m}$. Using the quotient rule for the Softmax function $\sigma(\cdot)$, we obtain:
\begin{equation}
    \frac{\partial\hat{\bz}_m}{\partial w_{j,m}} = \sum_{\alpha=1}^{m-1} (\delta_{j\alpha} - \sigma_j(\bw_m))\sigma_\alpha(\bw_m) \be_\alpha = \sigma_j(\bw_m) (\be_j - \hat{\bz}_m).
\end{equation}

\paragraph{Combined Gradient.}
Substituting Parts (I), (II), and (III) back into Eq.~\eqref{eq:gradient_chain_rule}, we obtain the final expression for the gradient:
\begin{equation}
    \label{eq:combined_gradient}
    \frac{\partial \mathcal{L}}{\partial w_{j,m}} = - \underbrace{\left( \mathbf{c}_t - \mathbb{E}_{\hat{\by}_m}[\mathbf{c}] \right)^\top}_{\text{Output Error}} \cdot \underbrace{\text{diag}(\phi'(\hat{\bz}_m))}_{\text{Element-wise Jacobian}} \cdot \underbrace{\sigma_j(\bw_m)(\bx_j - \hat{\bz}_m)}_{\text{Attention Gradient}}.
\end{equation}
This equation elucidates how the error signal backpropagates directly through the non-linear activation $\phi'$, and interacts with the attention mechanism's sensitivity to the input token $\bx_j$.

\paragraph{Detailed Decomposition.}
To analyze the specific learning dynamics, we substitute the precise functional forms into the combined gradient in Eq.~\eqref{eq:gradient_chain_rule}. Specifically, the three components are instantiated as follows:
\begin{itemize}
    \item \textbf{Output Error:} The general error term concretizes to $- (\mathbf{c}_t - \sum_{k=1}^K p_k \mathbf{c}_k)^\top$, representing the residual between the target class prototype $\mathbf{c}_t$ and the current probability-weighted prototypes.
    \item \textbf{Element-wise Jacobian:} By applying the Taylor expansion to the activation derivative $\phi'$, this term is approximated as $-2c \cdot \text{diag}( \frac{1}{m} \sum_{\alpha=1}^{m-1} \be_\alpha ) + 4c \cdot \text{diag}(\hat{\bz}_m^{\odot 3})$, where the linear term dominates the high-order residuals.
    \item \textbf{Attention Gradient:} Under the assumption of uniform attention initialization (where attention weights $\sigma_j(\bw_m) \approx 1/m$), this term simplifies to $\frac{1}{m} ( \be_j - \frac{1}{m} \sum_{\beta=1}^{m-1} \be_\beta )$, capturing the deviation of the specific token $\be_j$ from the historical context.
\end{itemize}
Multiplying these three specific terms and fully expanding the resulting expression leads to a sum of eight distinct gradient components:

\begin{align}
   \frac{\partial \mathcal{L}}{\partial w_{j,m}} 
   = &\frac{2c}{m} (1 - p_t)\mathbf{c}_t^\top \cdot \text{diag}\left( \frac{1}{m} \sum_{\alpha=1}^{m-1} \be_\alpha \right) \cdot \be_j \label{eq:split_1} \\
   &- \frac{2c}{m} (1 - p_t)\mathbf{c}_t^\top \cdot \text{diag}\left( \frac{1}{m} \sum_{\alpha=1}^{m-1} \be_\alpha \right) \cdot \left( \frac{1}{m} \sum_{\beta=1}^{m-1} \be_\beta \right) \label{eq:split_2} \\
   &- \frac{2c}{m} \left( \sum_{k \neq c} p_k \mathbf{c}_k \right)^\top \cdot \text{diag}\left( \frac{1}{m} \sum_{\alpha=1}^{m-1} \be_\alpha \right) \cdot \be_j \label{eq:split_3} \\
   &+ \frac{2c}{m} \left( \sum_{k \neq c} p_k \mathbf{c}_k \right)^\top \cdot \text{diag}\left( \frac{1}{m} \sum_{\alpha=1}^{m-1} \be_\alpha \right) \cdot \left( \frac{1}{m} \sum_{\beta=1}^{m-1} \be_\beta \right) \label{eq:split_4} \\
   &- \frac{4c}{m} (1 - p_t)\mathbf{c}_t^\top \cdot \text{diag}\left(\hat{\bz}_m^{\odot 3}\right) \cdot \be_j \label{eq:split_5} \\
   &+ \frac{4c}{m} (1 - p_t)\mathbf{c}_t^\top \cdot \text{diag}\left(\hat{\bz}_m^{\odot 3}\right) \cdot \left( \frac{1}{m} \sum_{\beta=1}^{m-1} \be_\beta \right) \label{eq:split_6} \\
   &+ \frac{4c}{m} \left( \sum_{k \neq c} p_k \mathbf{c}_k \right)^\top \cdot \text{diag}\left(\hat{\bz}_m^{\odot 3}\right) \cdot \be_j \label{eq:split_7} \\
   &- \frac{4c}{m} \left( \sum_{k \neq c} p_k \mathbf{c}_k \right)^\top \cdot \text{diag}\left(\hat{\bz}_m^{\odot 3}\right) \cdot \left( \frac{1}{m} \sum_{\beta=1}^{m-1} \be_\beta \right). \label{eq:split_8}
\end{align}

Leveraging the binary symmetry where $\mathbf{c}_1 = -\mathbf{c}_2$, the output error term simplifies significantly: $\mathbf{c}_t - \sum p_k \mathbf{c}_k = 2(1 - p_t)\mathbf{c}_t$. This allows us to merge the target and non-target components in Eqs.~\eqref{eq:split_1}-\eqref{eq:split_8}. Consequently, the eight-term expansion collapses into four concise terms, consisting of the linear and residual interactions with the target prototype:

\begin{equation}
\label{eq:gradient_final}
\begin{aligned}
   \frac{\partial \mathcal{L}}{\partial w_{j,m}} 
   &= \frac{4c}{m^2} (1 - p_t) \sum_{\alpha=1}^{m-1} \tc{\mathbf{c}_t, \be_\alpha, \be_j} \\
   &\quad - \frac{4c}{m^3} (1 - p_t) \sum_{\alpha=1}^{m-1} \sum_{\beta=1}^{m-1} \tc{\mathbf{c}_t, \be_\alpha, \be_\beta} \\
   &\quad - \frac{2}{m} (1 - p_t) \tc{\mathbf{c}_t, \hat{\bz}_m^{\odot 3}, \be_j} \\
   &\quad + \frac{2}{m^2} (1 - p_t) \sum_{\beta=1}^{m-1} \tc{\mathbf{c}_t, \hat{\bz}_m^{\odot 3}, \be_\beta}.
\end{aligned}
\end{equation}

\noindent This unified representation highlights the structural symmetry between the primary attention-like interactions (terms 1 and 2) and their corresponding higher-order corrections (terms 3 and 4), where the global context acts as a centering mechanism across both orders of approximation.

To account for the optimization over a batch of $n$ samples, we extend the gradient formulation by averaging over the sample index $i$. Let the superscript $(i)$ denote the specific variables (embeddings, probabilities, and residuals) associated with the $i$-th sample, and let $t_i$ represent the target class for that sample. The aggregate gradient becomes:

\begin{align}
   \frac{\partial \mathcal{L}}{\partial w_{j,m}} 
   &= \frac{4c}{n m^2} \sum_{i=1}^n \left[ (1 - p_{t_i}^{(i)}) \sum_{\alpha=1}^{m-1} \tc{\mathbf{c}_{t_i}, \be_\alpha^{(i)}, \be_j^{(i)}} \right]  \label{eq:simple_1} \\
   &\quad - \frac{4c}{n m^3} \sum_{i=1}^n \left[ (1 - p_{t_i}^{(i)}) \sum_{\alpha=1}^{m-1} \sum_{\beta=1}^{m-1} \tc{\mathbf{c}_{t_i}, \be_\alpha^{(i)}, \be_\beta^{(i)}} \right]  \label{eq:simple_2} \\
   &\quad - \frac{8c}{n m} \sum_{i=1}^n \left[ (1 - p_{t_i}^{(i)}) \tc{\mathbf{c}_{t_i}, \hat{\bz}_m^{\odot 3}, \be_j^{(i)}} \right]  \label{eq:simple_3} \\
   &\quad + \frac{8c}{n m^2} \sum_{i=1}^n \left[ (1 - p_{t_i}^{(i)}) \sum_{\beta=1}^{m-1} \tc{\mathbf{c}_{t_i}, \hat{\bz}_m^{\odot 3}, \be_\beta^{(i)}} \right]. \label{eq:simple_4}
\end{align}

\noindent Here, the outer summation $\frac{1}{n} \sum_{i=1}^n$ averages the contribution of each sample. For every sample $i$, the gradient is weighted by the model's uncertainty $(1 - p_{t_i}^{(i)})$, ensuring that well-classified samples contribute less to the parameter update, while the structural terms (interactions between target prototypes $\mathbf{c}_{t_i}$, history $\be_\alpha^{(i)}$, and current token $\be_j^{(i)}$) drive the learning process.

\textbf{Computing interaction strengths.}

We now analyze the first two terms, which represent the low-order interaction contributions to the gradient. We rely directly on the definitions of the signal set $\mathcal{S}_{2,m}$ and the noise set $\mathcal{I}_{2,m}$ established previously.

Based on Lemma \ref{thm:kappa}, the sample average of the contraction for any index tuple $(a,b)$ concentrates based on its membership in these sets:
\begin{equation}
\frac{1}{n}\sum_{i=1}^n \tc{\mathbf{c}_{t_i}, \be_a^{(i)}, \be_b^{(i)}} = 
\begin{cases} 
C_{\text{signal}}^{(2)} & \text{if } (a,b) \in \mathcal{S}_{2,m}, \\
\text{noise} \in [-\kappa, \kappa] & \text{if } (a,b) \in \mathcal{I}_{2,m},
\end{cases}
\end{equation}
where $\kappa = O(\sqrt{\frac{\log m}{n}})$ represents the noise floor.

1. Analyzing the inner sum of \eqref{eq:simple_1}:
Consider the term $\sum_{\alpha=1}^{m-1} \tc{\mathbf{c}_{t_i}, \be_\alpha^{(i)}, \be_j^{(i)}}$. We classify the contribution based on whether the fixed index $j$ can form a valid signal tuple with any $\alpha$.

\begin{itemize}
   \item \textbf{Case: $j$ is a relevant feature.} 
   If $j$ is part of the true parity support, there exists exactly one index $\alpha'$ such that the pair $(\alpha', j) \in \mathcal{S}_{2,m}$ (representing the complementary feature in the parity pair). For all other $\alpha \neq \alpha'$, the pair $(\alpha, j) \in \mathcal{I}_{2,m}$.
   
   \item \textbf{Case: $j$ is an irrelevant feature.} 
   If $j$ is not part of the support, then for all $\alpha \in [m-1]$, the pair $(\alpha, j)$ never matches the target parity. Thus, $(\alpha, j) \in \mathcal{I}_{2,m}$ for the entire summation.
\end{itemize}

The term $(1 - p_{t_i}^{(i)})$ acts as a scalar weight reflecting the model's uncertainty on sample $i$. Since this weight appears as a common factor for the gradient components of all indices $j$, it modulates the global magnitude but preserves the \textbf{relative structural difference} between relevant and irrelevant features. Consequently, for the purpose of asymptotic signal-to-noise analysis, we focus on the structural component of the summation (omitting the scalar $(1 - p_{t_i}^{(i)})$):

\begin{align}
\frac{1}{n} \sum_{i=1}^n \sum_{\alpha=1}^{m-1} \tc{\mathbf{c}_{t_i}, \be_\alpha^{(i)}, \be_j^{(i)}} 
&= \begin{cases}
C_{\text{signal}}^{(2)} + (m-2)\kappa & \text{if } \exists \alpha' : (\alpha', j) \in \mathcal{S}_{2,m}, \\
(m-1)\kappa & \text{otherwise.}
\end{cases} \label{eq:term1_bound}
\end{align}

2. Analyzing the inner sum of \eqref{eq:simple_2}:
This term involves a double summation over all pairs $(\alpha, \beta) \in [m-1] \times [m-1]$. This global term is independent of $j$.

\begin{itemize}
   \item The sum accumulates $C_{\text{signal}}^{(2)}$ for every tuple $(\alpha, \beta)$ that belongs to $\mathcal{S}_{2,m}$.
   
   \item Since the order of indices matters in the tensor contraction but parity is symmetric, $\mathcal{S}_{2,m}$ contains all permutations of the support indices. For a degree-2 parity, $|\mathcal{S}_{2,m}| = 2$.
   
   \item All other $(m-1)^2 - 2$ tuples belong to $\mathcal{I}_{2,m}$.
\end{itemize}

Thus, the bound is:
\begin{align}
\frac{1}{n} \sum_{i=1}^n \sum_{\alpha, \beta} \tc{\mathbf{c}_{t_i}, \be_\alpha^{(i)}, \be_\beta^{(i)}} 
&= |\mathcal{S}_{2,m}| \cdot C_{\text{signal}}^{(2)} + O(m^2 \kappa) \nonumber \\
&= 2C_{\text{signal}}^{(2)} + O(m^2 \kappa). \label{eq:term2_bound}
\end{align}

3. Combining the terms:

Substituting \eqref{eq:term1_bound} and \eqref{eq:term2_bound} back into the gradient expressions, and assuming $d < m \le 2d$, we unify the analysis by defining a generalized membership indicator dependent on the interaction order $r$.

Let $\mathbb{I}_j^{(r)}$ denote whether feature $j$ is an active component in the ground-truth signal set of order $r$, denoted $\mathcal{S}_{r,m}$:
\begin{equation}
\mathbb{I}_j^{(r)} := \mathbb{I}\left( \exists \, (i_1, \dots, i_r) \in \mathcal{S}_{r,m} \text{ such that } j \in \{i_1, \dots, i_r\} \right).
\end{equation}
This indicator effectively acts as a selector: $\mathbb{I}_j^{(r)} = 1$ if feature $j$ contributes to an $r$-th order interaction, and $0$ otherwise.

For our specific case where the leading interaction is second-order ($r=2$), the combined gradient expression becomes:

\begin{align}
\eqref{eq:simple_1} + \eqref{eq:simple_2} 
&\approx -\frac{4c}{m^2} \left( C_{\text{signal}}^{(2)} \cdot \mathbb{I}_j^{(2)} + O(m\kappa) \right) + \frac{4c}{m^3} \left( 2C_{\text{signal}}^{(2)} + O(m^2\kappa) \right) \\
&= -\frac{4c}{m^2} C_{\text{signal}}^{(2)} \cdot \mathbb{I}_j^{(2)} + \underbrace{\frac{8c}{m^3}C_{\text{signal}}^{(2)}}_{O(m^{-3})} + \underbrace{O\left(\frac{\kappa}{m}\right)}_{\text{noise}} \\
&= \begin{cases}
-\Theta\left( m^{-2} \right) \cdot C_{\text{signal}}^{(2)} & \text{if } \mathbb{I}_j^{(2)} = 1, \\
O\left( m^{-3} \right) \cdot C_{\text{signal}}^{(2)} + O\left( m^{-1}\kappa \right) & \text{if } \mathbb{I}_j^{(2)} = 0.
\end{cases}
\end{align}

4. Analyzing the inner sum of \eqref{eq:simple_3}:
This term involves a triple summation over indices $(k_1, k_2, k_3) \in [m]^3$ with one index $j$ fixed. We examine the tuple formed by $(k_1, k_2, k_3, j)$.

\begin{itemize}
    \item Case A (Relevant Feature): If $j$ is part of the ground-truth support set (i.e., $\mathbb{I}_j^{(4)} = 1$), then for the tuple to belong to $\mathcal{S}_{4,m}$, the indices $(k_1, k_2, k_3)$ must be a permutation of the remaining 3 support indices. The number of such permutations is $(4-1)! = 3! = 6$.
  
    \item Case B (Irrelevant Feature): If $j$ is not in the support set (i.e., $\mathbb{I}_j^{(4)} = 0$), no choice of $(k_1, k_2, k_3)$ can form a valid support tuple. The signal contribution is 0.
  
    \item The total number of terms in the summation is $m^3$.
\end{itemize}

The aggregated contribution is:
\begin{align}
\frac{1}{n} \sum_{i=1}^n \sum_{k_1, k_2, k_3} \tc{\mathbf{c}_{t_i}, \be_{k_1}^{(i)}, \be_{k_2}^{(i)}, \be_{k_3}^{(i)}, \be_j^{(i)}} 
&= \begin{cases}
6 C_{\text{signal}}^{(4)} + O(m^3 \kappa) & \text{if } \mathbb{I}_j^{(4)} = 1, \\
O(m^3 \kappa) & \text{if } \mathbb{I}_j^{(4)} = 0.
\end{cases} \label{eq:term3_bound}
\end{align}

5. Analyzing the inner sum of \eqref{eq:simple_4}:
This term involves a quadruple summation over all indices $(\beta, k_1, k_2, k_3)$. This global term is independent of $j$.

\begin{itemize}
   \item The sum accumulates $C_{\text{signal}}^{(4)}$ for every tuple $(k_1, k_2, k_3, \beta)$ that belongs to $\mathcal{S}_{4,m}$.
 
   \item Since the sum covers all possible combinations, it includes all valid permutations of the support set. For a degree-4 parity, $|\mathcal{S}_{4,m}| = 4! = 24$.
 
   \item The total number of terms in the summation is approximately $m^4$.
\end{itemize}

Thus, the bound is:
\begin{align}
\frac{1}{n} \sum_{i=1}^n \sum_{\beta, k_1, k_2, k_3} \tc{\mathbf{c}_{t_i}, \be_{k_1}^{(i)}, \be_{k_2}^{(i)}, \be_{k_3}^{(i)}, \be_\beta^{(i)}} 
&= 24 C_{\text{signal}}^{(4)} + O(m^4 \kappa). \label{eq:term4_bound}
\end{align}

6. Combining the 4th-order terms:

Substituting \eqref{eq:term3_bound} and \eqref{eq:term4_bound} back into the simplified gradient expressions:

\begin{align}
\eqref{eq:simple_3} + \eqref{eq:simple_4} 
&\approx -\frac{8c}{m^4} \left( 6 C_{\text{signal}}^{(4)} \cdot \mathbb{I}_j^{(4)} + O(m^3\kappa) \right) + \frac{8c}{m^5} \left( 24 C_{\text{signal}}^{(4)} + O(m^4\kappa) \right) \\
&= -\frac{48c}{m^4} C_{\text{signal}}^{(4)} \cdot \mathbb{I}_j^{(4)} + \underbrace{\frac{192c}{m^5}C_{\text{signal}}^{(4)}}_{O(m^{-5})} + \underbrace{O\left(\frac{\kappa}{m}\right)}_{\text{noise}} \\
&= \begin{cases}
-\Theta\left( m^{-4} \right) \cdot C_{\text{signal}}^{(4)} & \text{if } \mathbb{I}_j^{(4)} = 1, \\
O\left( m^{-5} \right) \cdot C_{\text{signal}}^{(4)} + O\left( m^{-1}\kappa \right) & \text{if } \mathbb{I}_j^{(4)} = 0.
\end{cases}
\end{align}

\textbf{Combined Gradient Analysis and Learning Dynamics}

Based on the detailed expansion of the second-order and fourth-order terms, we derive the total approximate gradient expectation. By simplifying the specific coefficients into asymptotic bounds, we highlight the scaling relationships with respect to the sequence length $m$.

\textbf{Total Gradient Approximation}

For any feature index $j$, the expected gradient can be expressed as a superposition of signal components (from relevant interaction orders) and noise/bias components. The total aggregated gradient is given by:

\begin{align}
\frac{\partial \mathcal{L}}{\partial w_{j,m}} &= 
\underbrace{-\Theta(m^{-2}) \cdot \mathbb{I}_j^{(2)}}_{\text{2nd-order Signal}} 
+ \underbrace{\Theta(m^{-3})}_{\text{2nd-order Bias}} \nonumber \\
&\quad 
\underbrace{-\Theta(m^{-4}) \cdot \mathbb{I}_j^{(4)}}_{\text{4th-order Signal}} 
+ \underbrace{\Theta(m^{-5})}_{\text{4th-order Bias}} 
+ \underbrace{O\left(\frac{\kappa}{m}\right)}_{\text{Noise}},
\label{eq:total_gradient_theta}
\end{align}

where $\mathbb{I}_j^{(r)} \in \{0, 1\}$ is the indicator function, which equals $1$ if and only if feature $j$ belongs to the ground-truth support set $\mathcal{S}_{r,m}$ of order $r$.

\subsection{Generalization to Arbitrary Activation Functions}
\label{sec:general_activation}
While the previous analysis utilized a specific quartic polynomial to demonstrate the emergence of second and fourth-order interaction signals, specific functional forms restrict the generality of the findings. To provide a comprehensive understanding of the Transformer's learning dynamics, we extend this framework to general non-linearities.
In this section, we generalize our derivation to arbitrary smooth activation functions using Taylor series expansions. This allows us to strictly quantify the signal intensity for interaction terms of any order $r$.
\textbf{Expansion Point Justification}:
The validity of the perturbative analysis depends on expanding the function around the operating point of the hidden states at initialization. Given the symmetric initialization of embeddings $\mathbf{e} \sim \mathcal{U}(-1, 1)$ and the linearity of the attention aggregation, the pre-activation feature $\hat{\bz}_m$ satisfies $\mathbb{E}[\hat{\bz}_m] = \mathbf{0}$. Consequently, the statistical behavior of the gradient is dictated by the local geometry of the activation function around zero. We therefore employ the Maclaurin series (Taylor expansion at $z=0$).
\textbf{Formal Definition}:
Let $\phi: \mathbb{R} \to \mathbb{R}$ be a smooth activation function. We express it as:
\begin{equation}
\phi(z) = \sum_{k=0}^{\infty} \frac{\phi^{(k)}(0)}{k!} z^k.
\end{equation}
Recall that the gradient flow is modulated by the derivative $\phi'(\hat{\bz}_m)$. Differentiating the series, we define the gradient scaling function as:
\begin{equation}
\label{eqn:phi_prime_expansion}
\phi'(z) = \sum_{k=1}^{\infty} \frac{\phi^{(k)}(0)}{(k-1)!} z^{k-1} = \sum_{r=1}^{\infty} \gamma_r z^{r-1},
\end{equation}
where $\gamma_r \triangleq \frac{\phi^{(r)}(0)}{(r-1)!}$ is the \textit{Interaction Coefficient} for order $r$.

\textbf{Generalized Gradient Decomposition}:
Substituting the series expansion of $\phi'(\cdot)$ into the chain rule, we decompose the gradient into a weighted sum of tensor contractions. This formulation reveals that the total gradient is a superposition of interaction terms across all orders $r$, where the $r$-th term is explicitly governed by the derivative coefficient $\gamma_r$:

\begin{align}
    \frac{\partial \mathcal{L}}{\partial w_{j,m}} 
    = \sum_{r=1}^{\infty} \gamma_r \cdot \Bigg\{ \nonumber 
    &\underbrace{ \frac{1}{n} \sum_{i=1}^n \left[ (1 - p_{t_i}^{(i)}) \cdot \tc{\mathbf{c}_{t_i}, (\hat{\bz}_m^{(i)})^{\odot (r-1)}, \be_j^{(i)}} \right] }_{\text{Term A: Current Token Interaction (Order } r \text{)}} \\
    - &\underbrace{ \frac{1}{n m} \sum_{i=1}^n \left[ (1 - p_{t_i}^{(i)}) \sum_{\beta=1}^{m} \tc{\mathbf{c}_{t_i}, (\hat{\bz}_m^{(i)})^{\odot (r-1)}, \be_\beta^{(i)}} \right] }_{\text{Term B: Context Bias Term (Order } r \text{)}} \Bigg\} 
\end{align}

This expansion highlights a hierarchical dependency on the sequence length $m$. Specifically, the interaction involving the current token $\be_j$ (Term A) scales as $O(m^{-r})$, while the context bias (Term B) is further suppressed by a factor of $1/m$ due to the averaging over the history.

Generalizing the asymptotic behavior observed in the quadratic and quartic cases, we can express the total gradient as a sum of signal and bias terms for each order $r$. The final generalized scaling law is:

\begin{equation}
  \frac{\partial \mathcal{L}}{\partial w_{j,m}} = 
  \sum_{r=1}^{\infty} \gamma_r \cdot \Bigg[
  \underbrace{-\Theta(m^{-r}) \cdot \mathbb{I}_j^{(r)} \cdot C_{\mathrm{signal}}^{(r)} }_{\text{Order-}r \text{ Signal}} 
  + \underbrace{\Theta(m^{-(r+1)}) \cdot C_{\mathrm{signal}}^{(r)} }_{\text{Order-}r \text{ Bias}} 
  \Bigg]
  + \underbrace{O\left( \frac{\kappa}{m} \right)}_{\text{Noise}}.
\end{equation}

where $\mathbb{I}_j^{(r)}$ represents the structured interaction tensor of order $r$ specific to token $j$. This confirms that lower-order terms (small $r$) dominate the learning dynamics, while higher-order terms decay exponentially with the complexity of the interaction.

\textbf{Conclusion and Implications}

The analytical result in \eqref{eq:total_gradient_theta} reveals the intrinsic learning dynamics of the model. We summarize the key findings as follows:

First, the model demonstrates a discriminative learning mechanism that effectively distinguishes between relevant and irrelevant features. For features within the support set (where $\mathbb{I}_j^{(r)}=1$), the gradient is dominated by a signal term scaling as $-\Theta(m^{-r})$, which provides a deterministic drive for the weights to capture the interaction. Conversely, for irrelevant features, this signal vanishes, leaving the gradient driven solely by higher-order bias terms and stochastic noise. This distinction allows the model to selectively amplify features in the support set while suppressing non-informative inputs.

Second, there exists a clear hierarchy in the learning difficulty of different interaction orders, governed by the sequence length $m$. The effective signal strength decays exponentially with the interaction order, dropping from $\Theta(m^{-2})$ for second-order terms to $\Theta(m^{-4})$ for fourth-order terms. Although higher-order terms are theoretically learnable, their significantly smaller gradient magnitudes compared to lower-order terms imply that the model exhibits a spectral bias, prioritizing the learning of lower-order interactions before fitting the more complex, high-order dependencies.

Third, the sample size $n$ plays a critical role in noise suppression, acting as the fundamental bottleneck for learning high-order interactions. Since the fourth-order signal intensity $\Theta(m^{-4})$ is extremely weak, the noise term $O(\kappa/m)$—where $\kappa$ is typically inversely proportional to $\sqrt{n}$—must be suppressed to a level lower than the signal. Consequently, learning high-order logic imposes a strict requirement for a sufficiently large sample size to reduce statistical variance; otherwise, the faint high-order signal will be overwhelmed by the noise derived from lower-order couplings, preventing the model from converging to the correct solution.

\section{Discussion on two latent token paradigms}
\label{sec:two_token_the}

Following the proof in Section \ref{sec:final_proof}, we need only discuss the capacity of $\bc_s$ to construct interaction terms of any order, either with the inputs alone or jointly with the output.

\subsection{Analysis of Implicit CoT: Independent Learnable Parameters}
\label{sec:analysis_independent_token}

We examine the value of the interaction term formed by an arbitrary subset of input tokens indexed by $J=(j_1, \dots, j_r)$ and the independent latent parameter $\bc_s$. The contraction is given by:
\begin{equation}
    T_{J, \bc_s} = \tc{x_{j_1} \bv, \dots, x_{j_r} \bv, \bc_s}.
\end{equation}
Exploiting the multilinear property of the tensor contraction, we factor out the scalar signs of the inputs:
\begin{equation}
    T_{J, \bc_s} = \left( \prod_{k=1}^r x_{j_k} \right) \cdot \tc{\bv, \dots, \bv, \bc_s}.
\end{equation}
Here, the geometric component $C_{\bc} := \tc{\bv, \dots, \bv, \bc_s}$ is a constant scalar determined solely by the initialization of $\bc_s$ and the base vector $\bv$. It does not vary across samples.

The statistical behavior of this term is governed by the product of the input signs $\chi_J = \prod_{k=1}^r x_{j_k}$. Since each input $x \in \{+1, -1\}$ is drawn from an independent uniform distribution (Rademacher distribution):
\begin{itemize}
    \item The product $\chi_J$ is also uniformly distributed over $\{+1, -1\}$.
    \item Consequently, for a single sample, the value of the contraction oscillates symmetrically:
    \begin{equation}
        T_{J, \bc_s}^{(i)} \sim \text{Uniform}(\{+C_{\bc}, -C_{\bc}\}).
    \end{equation}
\end{itemize}

Following the statistical analysis above, we observe that any interaction term involving the independent parameter $\bc_s$ is functionally equivalent to an irrelevant term ($\mathcal{I}_{r,m}$) in the standard expansion. Since $\bc_s$ contains no instance-specific sign information, it cannot cancel out the parity of the input subset and target. Consequently, the resulting contraction is a sum of zero-mean random variables.

By invoking the same concentration of measure arguments used in Lemma \ref{thm:kappa} (Hoeffding's inequality applied to bounded zero-mean variables), we formally bound the maximum contribution of these latent interactions.

\begin{lemma}[Concentration of Latent Interactions]\label{thm:kappa_latent}
Consider the set of interaction tuples $\mathcal{J}$ involving the independent latent token $\bc_s$. Under the assumption of i.i.d. uniform inputs over $\{\pm 1\}$, for any confidence level $1-p$, the maximum magnitude of the aggregated interaction strength is bounded by:
\begin{equation}
    \max_{J \in \mathcal{J}} \left| \frac{1}{n} \sum_{i=1}^n \tc{x_{j_1}^{(i)} \bv, \dots, x_{j_r}^{(i)} \bv, \bc_s} \right| \le \kappa_s,
\end{equation}
where $\kappa_s := C_{\bc} \cdot \sqrt{\frac{2}{n} \log \frac{2|\mathcal{J}|}{p}}$ represents the noise floor, and $C_{\bc}$ denotes the geometric magnitude of the contraction involving $\bc_s$.
\end{lemma}

\textbf{Implication.} This lemma establishes that interactions via the static token $\bc_s$ are mathematically indistinguishable from background noise. Unlike explicit intermediate tokens, which create a deterministic signal $C_{\text{signal}}$, the independent latent token fails to reduce the effective interaction order.

\subsection{Analysis of Implicit CoT: Dynamic Generation via Internal States}

To rigorously determine the signal-to-noise ratio, we analyze the joint interaction between the specific input subset, the latent state $\bh_m$, and the target output embedding.

\textbf{1. Expansion of the Joint Interaction Tensor}

Let $R = \{j_1, \dots, j_r\}$ be the set of indices for the relevant inputs. The target label is the parity $y_{m+1} = \prod_{j \in R} x_j$. We define the \textit{Joint Interaction Logic} $\Psi_R$ as the contraction of the input query, the latent thought, and the target:
\begin{equation}
    \Psi_R = \tc{ x_{j_1}\bv, \dots, x_{j_r}\bv, \bh_m, y_{m+1}\bu }.
\end{equation}
Substituting the Taylor expansion $\bh_m = \sum_{\ell=0}^{\infty} c_\ell \left( \frac{1}{m} \sum_{k=1}^m x_k \bv \right)^{\odot \ell}$, we expand the tensor into a sum over all possible interaction orders:
\begin{equation}
\label{eq:expansion_master}
    \Psi_R = \sum_{\ell=0}^{\infty} \frac{c_\ell}{m^\ell} \sum_{k_1=1}^m \dots \sum_{k_\ell=1}^m \underbrace{ \tc{ x_{j_1}\bv, \dots, x_{j_r}\bv, x_{k_1}\bv, \dots, x_{k_\ell}\bv,  y_{m+1} \bu } }_{\text{Elementary Interaction Term } \mathcal{T}}.
\end{equation}

Here, strictly speaking, $\bh_m$ cannot assist the explicit context in forming a \textit{stable} signal for the target logic. Since $\bh_m$ aggregates the global information of the entire sequence (as a dense sum), it violates the ``minimal subset'' constraint required for the definition of an $R$-order parity interaction, effectively acting as a source of interference rather than a precise guide.

However, the expansion reveals that $\bh_m$ inherently contains high-order terms of the aggregated context $\hat{\bz}_m$ (where $\ell \ge r$). This theoretically augments the model's capacity to capture complex, higher-order interaction signals that linear components cannot represent. Nevertheless, mining these valid high-order signals from the vast combinatorial space incurs a prohibitive training cost. The model requires an exponential number of samples and extended training duration to filter out the noise and converge on the correct high-order logic.

\section{Data Sample}
\label{sec:data_example}

Table A.1 presents a complete sample from the \textbf{NatBool-DAG} dataset. This instance features a \textbf{reasoning depth of 7} within the \textit{Medical Diagnosis} theme. The content within the box below illustrates the exact textual format fed to the model, alongside the rigorous step-by-step derivation used for supervision. 
\begin{table}[h!]
    \centering
    \caption{NatBool-DAG Sample Instance (Theme: Medical Diagnosis, Depth: 7)}
    \label{tab:sample_data}
    \begin{databox}[NatBool-DAG Sample ID: 0c63eebe-9289-4632]
    \small
    \textbf{[METADATA]} \\
    \textbf{Theme:} Medical Diagnosis \quad|\quad \textbf{Reasoning Steps:} 6 \quad|\quad \textbf{Total Nodes:} 13
    
    \noindent\rule{\textwidth}{0.4pt}
    
    \textbf{[INSTRUCTION]} \\
    Based on the observed data and system rules, determine the status of 'Urgent Surgery Needed'. Is it True or False?
    
    \vspace{0.5em}
    \textbf{[INPUT CONTEXT]} \\
    === SCENARIO: Medical Diagnosis ===
    
    \textit{--- OBSERVED DATA ---} \\
    {[Lab Report]} 'Blood Test A Type-2': NOT DETECTED/NORMAL. \\
    {[Lab Report]} 'Critical Genetic Marker': DETECTED/HIGH. \\
    {[Lab Report]} 'Suppressed Blood Test A': DETECTED/HIGH. \\
    {[Lab Report]} 'Routine Previous History': NOT DETECTED/NORMAL. \\
    {[Lab Report]} 'MRI Scan Variant-X': NOT DETECTED/NORMAL. \\
    {[Lab Report]} 'Previous History': NOT DETECTED/NORMAL. \\
    {[Lab Report]} 'Elevated Patient Fever': NOT DETECTED/NORMAL.
    
    \textit{--- SYSTEM RULES ---} \\
    - Symptom Check: Suspect 'Urgent Surgery Needed' if patient shows 'Viral Load Type-1' OR 'Critical Enzyme Level'. \\
    - Differential: 'Viral Load Type-1' is indicated if 'Cell Regeneration Type-2' contradicts 'Cell Regeneration Type-2'. \\
    - Protocol: Confirm 'Cell Regeneration Type-2' only if 'Immune Response Type-1' AND 'Immune Response Type-1' are present. \\
    - Protocol: Confirm 'Secondary Viral Load' only if 'Critical Genetic Marker' AND 'Previous History' are present. \\
    - Exclusion: 'Abnormal Cell Regeneration' is ruled out (False) only if both 'Secondary Viral Load' and 'Secondary Viral Load' are True. \\
    - Differential: 'Critical Enzyme Level' is indicated if 'Cell Regeneration Type-2' contradicts 'Cell Regeneration Type-2'. \\
    - Differential: 'Immune Response Type-1' is indicated if 'Neural Activity' contradicts 'Abnormal Cell Regeneration'. \\
    - Exclusion: 'Neural Activity' is ruled out (False) only if both 'Secondary Viral Load' and 'Secondary Viral Load' are True.
    
    \noindent\rule{\textwidth}{0.4pt}
    
    \textbf{[GROUND TRUTH CHAIN-OF-THOUGHT]} \\
    1. Check status of 'Critical Genetic Marker': The logs indicate it is True. \\
    2. Check status of 'Previous History': The logs indicate it is False. \\
    Analyzing 'Secondary Viral Load': \\
    \hspace*{1em} - Requires: Critical Genetic Marker (True) AND Previous History (False) \\
    \hspace*{1em} - Logic: AND logic evaluates to False. $\rightarrow$ 'Secondary Viral Load' is False. \\
    Analyzing 'Neural Activity': \\
    \hspace*{1em} - Requires: Secondary Viral Load (False) NAND Secondary Viral Load (False) \\
    \hspace*{1em} - Logic: NAND logic evaluates to True. $\rightarrow$ 'Neural Activity' is True. \\
    Analyzing 'Abnormal Cell Regeneration': \\
    \hspace*{1em} - Requires: Secondary Viral Load (False) NAND Secondary Viral Load (False) \\
    \hspace*{1em} - Logic: NAND logic evaluates to True. $\rightarrow$ 'Abnormal Cell Regeneration' is True. \\
    Analyzing 'Immune Response Type-1': \\
    \hspace*{1em} - Requires: Neural Activity (True) XOR Abnormal Cell Regeneration (True) \\
    \hspace*{1em} - Logic: XOR logic evaluates to False. $\rightarrow$ 'Immune Response Type-1' is False. \\
    Analyzing 'Cell Regeneration Type-2': \\
    \hspace*{1em} - Requires: Immune Response Type-1 (False) AND Immune Response Type-1 (False) \\
    \hspace*{1em} - Logic: AND logic evaluates to False. $\rightarrow$ 'Cell Regeneration Type-2' is False. \\
    Analyzing 'Viral Load Type-1': \\
    \hspace*{1em} - Requires: Cell Regeneration Type-2 (False) XOR Cell Regeneration Type-2 (False) \\
    \hspace*{1em} - Logic: XOR logic evaluates to False. $\rightarrow$ 'Viral Load Type-1' is False. \\
    Analyzing 'Critical Enzyme Level': \\
    \hspace*{1em} - Requires: Cell Regeneration Type-2 (False) XOR Cell Regeneration Type-2 (False) \\
    \hspace*{1em} - Logic: XOR logic evaluates to False. $\rightarrow$ 'Critical Enzyme Level' is False. \\
    Analyzing 'Urgent Surgery Needed' (Target): \\
    \hspace*{1em} - Requires: Viral Load Type-1 (False) OR Critical Enzyme Level (False) \\
    \hspace*{1em} - Logic: OR logic evaluates to False. $\rightarrow$ 'Urgent Surgery Needed' is False.
    
    \vspace{0.5em}
    \textbf{Final Answer: False}

    \noindent\rule{\textwidth}{0.4pt}
    \textbf{[OUTPUT]} \\
    False \\

    \end{databox}
\end{table}

\subsection{Implementation Details}
\label{sec:experiments_details}
\paragraph{Training Setup.}
We perform \textbf{full-parameter fine-tuning} to ensure the model fully adapts to the latent reasoning mechanism. The training is conducted for $10$ epochs with a learning rate of $1 \times 10^{-5}$, utilizing FP16 mixed-precision optimization. The maximum sequence length is set to $2048$ tokens to accommodate the full context of logical reasoning paths. To ensure stable convergence, we configure the per-device batch size to $1$ with $8$ gradient accumulation steps.
\paragraph{Method Hyperparameters.}
For our proposed ALiCoT, the latent tokens are initialized as independent learnable parameters. We set the sequence length of these latent tokens to $20$. The coefficient for the alignment loss ($\lambda$ in Eq. \ref{eq:alignment_loss}) is set to $0.5$, balancing the primary generation task with the semantic alignment of the latent space.

\section{Main Experimental Results}
\label{sec:main_experimental_results}
We conduct a comprehensive evaluation of ALiCoT across diverse model architectures and scales, including Llama-3.2 (1B, 3B) and Qwen3 (0.6B, 4B). Crucially, for all evaluated methods, we enforce extreme compression, where the entire chain of intermediate reasoning steps is fully suppressed to test the absolute limits of implicit generation. Our analysis primarily focuses on the comparison against Imp. Base-1. While we report Imp. Base-2 for the Qwen3-0.6B model, it is excluded for larger models due to computational resource constraints. To ensure a fair and rigorously controlled comparison, all experiments utilize identical training datasets and iterations. The detailed results are presented in Table \ref{tab:refined_results}.

\begin{table*}[t]
    \centering
    \small 
    \setlength{\tabcolsep}{3pt} 
    \renewcommand{\arraystretch}{1.15}
    \caption{\textbf{Scalability Analysis (3--10 Hops).} Comparison of \textbf{ALiCoT} against baselines across model scales. \textbf{Explicit CoT} serves as the global reference with 100\% accuracy. Note that \textbf{Imp. Base-2} is only reported for Qwen3 (0.6B) as it failed to converge on larger models.}
    \label{tab:refined_results}
    
    \begin{tabular*}{\textwidth}{l @{\extracolsep{\fill}} l ccccccccc}
        \toprule
        \multirow{2}{*}{\textbf{Model Setting}} & \multirow{2}{*}{\textbf{Method}} & \multicolumn{8}{c}{\textbf{Accuracy (\%) by Hops}} & \textbf{Avg.} \\
        \cmidrule(lr){3-10} 
         & & \textbf{3} & \textbf{4} & \textbf{5} & \textbf{6} & \textbf{7} & \textbf{8} & \textbf{9} & \textbf{10} & \textbf{Acc} \\
        \midrule

        \textit{Reference} & Expl. CoT & 100.00 & 100.00 & 100.00 & 100.00 & 100.00 & 100.00 & 100.00 & 100.00 & 100.00 \\
        \midrule
        
        \multirow{2}{*}{\textbf{Llama-3.2 (1B)}} 
          & Imp. Base-1  & \textbf{68.18} & \textbf{69.57} & 71.01 & 68.77 & 66.20 & \textbf{66.67} & 72.05 & 63.84 & 68.50 \\
          & \textbf{ALiCoT} & 65.91 & 66.89 & \textbf{72.46} & \textbf{73.68} & \textbf{74.56} & 65.08 & \textbf{76.42} & \textbf{74.01} & \textbf{71.25} \\
        \cmidrule(lr){1-11} 
        
        \multirow{2}{*}{\textbf{Llama-3.2 (3B)}} 
          & Imp. Base-1  & 65.91 & 66.89 & 72.46 & 73.68 & 74.56 & 65.08 & 76.42 & 74.01 & 71.25 \\
          & \textbf{ALiCoT} & \textbf{93.64} & \textbf{88.96} & \textbf{84.42} & \textbf{88.07} & \textbf{82.58} & \textbf{85.71} & \textbf{89.52} & \textbf{81.36} & \textbf{86.85} \\

        \midrule
        
        \multirow{3}{*}{\textbf{Qwen3 (0.6B)}} 
          & Imp. Base-1  & 68.64 & 68.56 & 72.46 & 72.63 & 72.47 & 70.37 & 75.11 & 70.62 & 71.41 \\
          & Imp. Base-2  & 33.48 & 38.15 & 42.50 & 46.52 & 47.91 & 50.52 & 55.60 & 53.80 & 45.87 \\
          & \textbf{ALiCoT} & \textbf{87.27} & \textbf{84.28} & \textbf{82.25} & \textbf{85.26} & \textbf{82.23} & \textbf{82.01} & \textbf{84.72} & \textbf{83.62} & \textbf{83.94} \\
        \cmidrule(lr){1-11}
        
        \multirow{2}{*}{\textbf{Qwen3 (4B)}} 
          & Imp. Base-1  & 79.09 & 75.92 & 75.36 & 81.05 & 77.35 & 78.84 & 81.22 & 74.01 & 77.88 \\
          & \textbf{ALiCoT} & \textbf{100.00} & \textbf{97.99} & \textbf{96.74} & \textbf{93.68} & \textbf{92.68} & \textbf{92.59} & \textbf{92.58} & \textbf{92.66} & \textbf{95.01} \\
        \cmidrule(lr){1-11}
        
        \multirow{2}{*}{\textbf{Qwen3 (8B)}} 
          & Imp. Base-1  & 77.73 & 77.26 & 78.62 & 81.75 & 78.05 & 80.42 & 84.72 & 78.53 & 79.56 \\
          & \textbf{ALiCoT} & \textbf{100.00} & \textbf{96.66} & \textbf{95.29} & \textbf{96.49} & \textbf{92.33} & \textbf{92.59} & \textbf{87.77} & \textbf{92.09} & \textbf{94.34} \\
        
        \bottomrule
    \end{tabular*}
\end{table*}

\textbf{Superior Accuracy under Extreme Compression.}
ALiCoT consistently outperforms the baselines across all model architectures and scales, demonstrating remarkable resilience to the extreme compression of intermediate reasoning steps. While \textit{Imp. Base-1} struggles to maintain accuracy—particularly evident in the Qwen3-0.6B and Llama-3.2-3B settings where it lags significantly—ALiCoT achieves substantial gains. For instance, on Qwen3 (4B), ALiCoT improves the average accuracy by over \textbf{17\%} compared to Imp. Base-1 (95.01

\textbf{Robustness in Long-Reasoning.}
ALiCoT exhibits superior stability as the reasoning depth increases (from 3 to 10 hops). In the challenging 10-hop scenarios, where baseline performance typically degrades due to error accumulation in the implicit state, ALiCoT maintains high proficiency. Notably, in the Llama-3.2 (1B) setting, although the baseline performs competitively on shorter hops, ALiCoT surpasses it on longer chains (e.g., +10.17\% on Hop 10), indicating that our method effectively internalizes complex multi-step logic without suffering from "reasoning collapse."

\textbf{Impact of Latent Noise and Context Length.}
As demonstrated in Section \ref{sec:two_token_the}, autoregressive latent tokens (Imp. Base-2) introduce significantly higher noise levels compared to independent learnable parameters (Imp. Base-1 and ALiCoT). Given identical latent lengths, we observe divergent behaviors driven by this noise:

Imp. Base-1 (Low Noise): Due to the minimal interference from independent initialization, the model is not easily distracted. Consequently, it follows the standard complexity trend where performance is higher on simpler tasks (Short Hops) and naturally decays as reasoning complexity increases.

Imp. Base-2 (High Noise): The autoregressive noise creates a dependency on input context length for robustness.

\textit{Vulnerability in Short Contexts:} In simple tasks (e.g., Hop 3), the input context is sparse. Lacking sufficient information to "anchor" its attention, the model becomes highly susceptible to the noise generated by the latent tokens, leading to performance collapse (33.48
\textit{Robustness in Long Contexts:} Conversely, as the reasoning chain lengthens (e.g., Hop 10), the richer context provides stronger supervision signals. This extended information helps the model filter out latent noise, resulting in a counter-intuitive performance recovery (53.80
This contrast highlights that ALiCoT’s independent initialization is crucial for stabilizing inference, especially for tasks with varying context lengths.

\end{document}